\documentclass[12pt]{article}
\usepackage{setspace}
\RequirePackage{amsthm,amsmath,amsfonts,amssymb}
\usepackage{float}

\usepackage[utf8]{inputenc}
\usepackage{bbm}
\usepackage{bm} 
\usepackage{dsfont}
\usepackage{paralist}
\usepackage{enumitem}
\usepackage[usenames]{color}
\usepackage{url}
\usepackage[colorlinks,
linkcolor=red,
anchorcolor=blue,
citecolor=blue
]{hyperref}
\usepackage{xspace,exscale,relsize}
\usepackage{fancybox,shadow}
\usepackage{graphicx}
\usepackage{caption}
\usepackage{subcaption} 
\usepackage{xcolor}
\usepackage[title]{appendix}
\usepackage{caption}
\usepackage{extarrows}
\usepackage{comment}
\usepackage{booktabs} 
\usepackage{makecell, multirow}

\usepackage{mathrsfs}
\makeatletter
\let\over=\@@over \let\overwithdelims=\@@overwithdelims
\let\atop=\@@atop \let\atopwithdelims=\@@atopwithdelims
\let\above=\@@above \let\abovewithdelims=\@@abovewithdelims
\makeatother
\usepackage{rotating}

	\usepackage{ifpdf}
	
	\usepackage{psfrag}
	
	\usepackage{prettyref}

	\usepackage{tikz}
	\usetikzlibrary{arrows}
	\tikzstyle{int}=[draw, fill=blue!20, minimum size=2em]
	\tikzstyle{dot}=[circle, draw, fill=blue!20, minimum size=2em]
	\tikzstyle{init} = [pin edge={to-,thin,black}]

	\usepackage[all]{xy}

	\usepackage{algorithm}
	\usepackage{algpseudocode}
	
	\usepackage{mathtools}

	\def\ba{{\boldsymbol a}}   \def\bA{\boldsymbol A}  
	\def\bb{{\boldsymbol b}}   \def\bB{\boldsymbol B}  
	     
	\def\bd{{\boldsymbol d}}     
	     \def\EE{\mathbb{E}}
	\def\bff{{\boldsymbol f}}   
	     \def\GG{\mathbb{G}}
	   \def\bH{\boldsymbol H}

	     \def\PP{\mathbb{P}}

	\def\bu{{\boldsymbol u}}     
	\def\bv{\boldsymbol v}   \def\bV{\boldsymbol V}  
	   \def\bW{\boldsymbol W}  
	\def\bx{\boldsymbol x}     
	     
	\def\bz{\boldsymbol z}     
	
	\def\11{\mathbbm{1}}
	
	\def\calA{{\cal  A}} 
	\def\calB{{\cal  B}} 
	\def\calC{{\cal  C}} 
	\def\calD{{\cal  D}} 
	\def\calE{{\cal  E}} 
	\def\calF{{\cal  F}} 
	\def\calG{{\cal  G}} 
	\def\calH{{\cal  H}} 
	\def\calI{{\cal  I}} 
	\def\calJ{{\cal  J}} \def\cJ{{\cal  J}}
	 
	\def\calL{{\cal  L}} 
	 
	\def\calN{{\cal  N}} 
	 
	\def\calP{{\cal  P}} \def\cP{{\cal  P}}
	 
	\def\calR{{\cal  R}}

	\def\calU{{\cal  U}}

	\newcommand{\bfsym}[1]{\ensuremath{\boldsymbol{#1}}}

	\def\btheta{{\bfsym {\theta}}}           \def\bTheta {\bfsym {\Theta}}
	          
	             \def\bSigma{\bfsym \Sigma}
	\def\blambda {\bfsym {\lambda}}

	\DeclareMathOperator{\argmin}{argmin}

	\DeclareMathOperator{\sgn}{sgn}
	
	\DeclareMathOperator{\Var}{Var}

	\def\Tr{{{\rm Tr}}}

	\graphicspath {{fig/}}
	
	\usepackage{ifthen}
	\newboolean{aos}
	\setboolean{aos}{FALSE}

	\ifx\eqref\undefined
	\newcommand{\eqref}[1]{~(\ref{#1})}
	\fi
	\ifx\mod\undefined
	\def\mod{\mathop{\rm mod}}
	\fi
	
	\usepackage{bm}

	\def\argmin{\mathop{\rm argmin}}
	
	\def\dim{\mathop{\rm dim}}
	
	\def\exp{\mathop{\rm exp}}

	\def\EE{\Expect}

	\def\Var{\mathrm{Var}}

	\def\PP{\mathbb{P}}

	\def\eqdef{\triangleq}

	\def\simiid{\stackrel{iid}{\sim}}
	
	\newcommand{\abs}[1]{\left| #1 \right|}

	\makeatletter
	\def\bbordermatrix#1{\begingroup \m@th
		\@tempdima 4.75\p@
		\setbox\z@\vbox{%
			\def\cr{\crcr\noalign{\kern2\p@\global\let\cr\endline}}%
			\ialign{$##$\hfil\kern2\p@\kern\@tempdima&\thinspace\hfil$##$\hfil
				&&\quad\hfil$##$\hfil\crcr
				\omit\strut\hfil\crcr\noalign{\kern-\baselineskip}%
				#1\crcr\omit\strut\cr}}%
		\setbox\tw@\vbox{\unvcopy\z@\global\setbox\@ne\lastbox}%
		\setbox\tw@\hbox{\unhbox\@ne\unskip\global\setbox\@ne\lastbox}%
		\setbox\tw@\hbox{$\kern\wd\@ne\kern-\@tempdima\left[\kern-\wd\@ne
			\global\setbox\@ne\vbox{\box\@ne\kern2\p@}%
			\vcenter{\kern-\ht\@ne\unvbox\z@\kern-\baselineskip}\,\right]$}%
		\null\;\vbox{\kern\ht\@ne\box\tw@}\endgroup}
	\makeatother
	
	\newcommand{\ceil}[1]{{\left\lceil {#1} \right \rceil}}

	\newcommand{\reals}{\mathbb{R}}
	\newcommand{\naturals}{\mathbb{N}}

	\newcommand{\Expect}{\mathbb{E}}

	\newcommand{\pth}[1]{\left( #1 \right)}
	\newcommand{\qth}[1]{\left[ #1 \right]}
	\newcommand{\sth}[1]{\left\{ #1 \right\}}

	\newcommand{\Binom}{\text{Binom}}

	\newcommand{\indc}[1]{{\mathbf{1}_{\left\{{#1}\right\}}}}

	\definecolor{myblue}{rgb}{.8, .8, 1}
	\definecolor{mathblue}{rgb}{0.2472, 0.24, 0.6} 
	\definecolor{mathred}{rgb}{0.6, 0.24, 0.442893}
	\definecolor{mathyellow}{rgb}{0.6, 0.547014, 0.24}

	\newcommand{\barr}{{\bar{r}}}

	\newrefformat{eq}{(\ref{#1})}
	\newrefformat{thm}{Theorem~\ref{#1}}
	\newrefformat{th}{Theorem~\ref{#1}}
	\newrefformat{chap}{Chapter~\ref{#1}}
	\newrefformat{sec}{Section~\ref{#1}}
	\newrefformat{seca}{Section~\ref{#1}}
	\newrefformat{algo}{Algorithm~\ref{#1}}
	\newrefformat{asmp}{Assumption~\ref{#1}}
	\newrefformat{fig}{Fig.~\ref{#1}}
	\newrefformat{tab}{Table~\ref{#1}}
	\newrefformat{rmk}{Remark~\ref{#1}}
	\newrefformat{clm}{Claim~\ref{#1}}
	\newrefformat{def}{Definition~\ref{#1}}
	\newrefformat{cor}{Corollary~\ref{#1}}
	\newrefformat{lmm}{Lemma~\ref{#1}}
	\newrefformat{prop}{Proposition~\ref{#1}}
	\newrefformat{pr}{Proposition~\ref{#1}}
	\newrefformat{property}{Property~\ref{#1}}
	\newrefformat{app}{Appendix~\ref{#1}}
	\newrefformat{apx}{Appendix~\ref{#1}}
	\newrefformat{ex}{Example~\ref{#1}}
	\newrefformat{exer}{Exercise~\ref{#1}}
	\newrefformat{soln}{Solution~\ref{#1}}
	
	\def\unifto{\mathop{{\mskip 3mu plus 2mu minus 1mu%
				\setbox0=\hbox{$\mathchar"3221$}%
				\raise.6ex\copy0\kern-\wd0%
				\lower0.5ex\hbox{$\mathchar"3221$}}\mskip 3mu plus 2mu minus 1mu}}

	\ifx\lesssim\undefined
	\def\simleq{{{\mskip 3mu plus 2mu minus 1mu%
				\setbox0=\hbox{$\mathchar"013C$}%
				\raise.2ex\copy0\kern-\wd0%
				\lower0.9ex\hbox{$\mathchar"0218$}}\mskip 3mu plus 2mu minus 1mu}}
	\else
	\def\simleq{\lesssim}
	\fi
	
	\ifx\gtrsim\undefined
	\def\simgeq{{{\mskip 3mu plus 2mu minus 1mu%
				\setbox0=\hbox{$\mathchar"013E$}%
				\raise.2ex\copy0\kern-\wd0%
				\lower0.9ex\hbox{$\mathchar"0218$}}\mskip 3mu plus 2mu minus 1mu}}
	\else
	\def\simgeq{\gtrsim}
	\fi

		\theoremstyle{definition}
		
		\newif\ifmapx
		{\catcode`/=0 \catcode`\\=12/gdef/mkillslash\#1{#1}}
		\edef\jobnametmp{\expandafter\string\csname ic_apx\endcsname}
		\edef\jobnameapx{\expandafter\mkillslash\jobnametmp}
		\edef\jobnameexpand{\jobname}
		\ifx\jobnameexpand\jobnameapx
		\mapxtrue
		\else
		\mapxfalse
		\fi

		\renewcommand{\hat}{\widehat}
		\renewcommand{\tilde}{\widetilde}

		\newcommand{\fidd}{{FIDDLE}}
		\newcommand{\aipw}{{\sf AIPW}}

		\newcommand{\mfunc}{\mu_1}

		\newcommand{\relu}{{\text{ReLU}~}}
		
		\newcommand{\opt}{\mathsf{opt}}
		
		\newcommand{\fast}{{\sf{FAST}~}}

		\newcommand{\mufidd}{{\hat \mu^{\sf{FIDDLE}}}}

		\newtheorem{theorem}{Theorem}
		\newtheorem{lemma}{Lemma}

		\newtheorem{definition}{Definition}
		
		\theoremstyle{definition}
		
		\newtheorem{remark}{Remark}
		
		\newtheorem{myassumption}{Assumption}
\usepackage{physics} 

\usepackage{natbib}
\usepackage[margin=1in]{geometry}  

\newcommand{\blind}{1}

\begin{document}

%


\if1\blind
{

  \title{\bf Factor Informed Double Deep Learning For Average Treatment Effect Estimation}
  \author{Jianqing Fan, Soham Jana, Sanjeev Kulkarni, and Qishuo Yin \thanks{J.F. and S.K. are with the Department of Operations Research and Financial Engineering and Department of Electric and Computer Engineering, Princeton University, Princeton, NJ, USA email: \url{jqfan@princeton.edu}, \url{kulkarni@princeton.edu}. Q.Y. is with the Department of Operations Research and Financial Engineering, Princeton University, Princeton, NJ, USA email: \url{qy1448@princeton.edu}. S.J. is with the Department of Applied and Computational Mathematics and Statistics, University of Notre Dame, Notre Dame, IN, USA, (correspondence to: \url{soham.jana@nd.edu}). J.F.'s research is supported by  NSF Grants DMS-2210833 and DMS-2412029 and the ONR Grant N00014-25-1-2317. }}
  \maketitle
} \fi

\if0\blind
{
  \bigskip
  \bigskip
  \bigskip
  \begin{center}
    {\LARGE\bf Factor Informed Double Deep Learning For Average Treatment Effect Estimation}
\end{center}
  \medskip
} \fi

\begin{abstract}

We investigate the problem of estimating the average treatment effect (ATE) under a very general setup where the covariates can be high-dimensional, highly correlated, and can have sparse nonlinear effects on the propensity and outcome models. We present the use of a Double Deep Learning strategy for estimation, which involves combining recently developed factor-augmented deep learning-based estimators, FAST-NN,  for both the response functions and propensity scores to achieve our goal.  By using FAST-NN, our method can select variables that contribute to propensity and outcome models in a completely nonparametric and algorithmic manner and adaptively learn low-dimensional function structures through neural networks.  Our proposed novel estimator, FIDDLE (Factor Informed Double Deep Learning Estimator), estimates ATE based on the framework of augmented inverse propensity weighting AIPW with the FAST-NN-based response and propensity estimates. FIDDLE consistently estimates ATE even under model misspecification, and is flexible to also allow for low-dimensional covariates. Our method achieves semiparametric efficiency under a very flexible family of propensity and outcome models. We present extensive numerical studies on synthetic and real datasets to support our theoretical guarantees and establish the advantages of our methods over other traditional choices, especially when the data dimension is large.

\end{abstract}

\noindent%




\section{Introduction}
Estimating the average treatment effect ({ATE}) is a central task in causal inference, which has led to significant findings in many disciplines, including economics \citep{oreopoulos2006estimating} and political science \citep{aronow2013beyond}. ATE measures the expected difference in responses between units assigned to a treatment and those assigned to a control. In mathematical terms, given an experimental unit with covariate vector $\bx\in \reals^p$ and treatment assignment indicator $T$ ($T=1$ denotes that the unit receives treatment and $T=0$ indicates that the unit is in the control group), the population outcome of $y$ is given as
\begin{align}\label{eq:model-0}
	\EE[y|T,\bx] = \mu_0^*(\bx)\indc{T=0}+\mu_1^*(\bx)\indc{T=1},
\end{align}
where $\mu_0^*,\mu_1^*$ are unknown outcome functions and the ATE is given by  $\mu = \EE\qth{\mu_1^*(\bx)-\mu_0^*(\bx)}.$
In practice, it is common to collect data on many variables deemed important in affecting policy outcomes; however, the treatment effect changes depending only on a handful of covariates, which are usually unknown to statisticians.   Our method allows us to select important variables in a completely nonparametric and algorithmic manner.  In addition, the covariates affecting the responses and treatments can be high-dimensional and highly correlated, and researchers might have incorrect assumptions about the data-generating models on the outcome and propensity functions.  These challenges can be addressed by employing the recently developed neural network method FAST-NN  \citep[Factor-Augmented Sparse Throughput Neural Networks]{fan2024factor}.

The ATE estimation problem becomes significantly challenging when the covariate dimension $p$ grows with the sample size, and could be significantly larger than that. It is standard in such scenarios to assume that the output functionals are low-dimensional functions of $\bx$. Factor modeling is historically considered to be an excellent choice for studying low-dimensional structures in the data  and can also account for dependency among data variables
\citep{fama2015five, fan2020factor}. 
Applications exist in various important statistical problems, such as covariance estimation \citep{fan2008high}, dependence modeling  
\citep{oh2017modeling}, variable selection \citep{fan2020factor}, tensor modeling \citep{zhou2025factor} and clustering \citep{tang2024factor}. Given a $p$-dimensional random vector $\bx$, the factor model assumes
\begin{align}\label{eq:factor}
	\bx=\bB\bff + \bu, \quad \bB\in \reals^{p\times r},\bff\in \reals^r, \bu\in \reals^p,\quad r<p, 
\end{align}
where the latent factor $\bff$ and the independent idiosyncratic part $\bu$ are unobservable random variables and the factor loading matrix $\bB\in\reals^{p\times r}$ is fixed but unknown. The loading matrix $\bB$ indicates how the covariate vector $\bx$ depends on the latent factor $\bff$. 

For modeling a function $m(\bx)$ such as the outcome or propensity functions using low-dimensional components, one often assumes the factor-augmented sparse throughput ({FAST}) model $m(\bx)\eqdef m(\bff,\bu_{\calJ})$, where $\calJ$ is the set of active coordinates of $\bu$ and $r+|\calJ|$ (here and below, given any set $\calJ$, we denote its size by $|\calJ|$) is significantly smaller than the covariate dimension $p$. In particular, the factor structure enables us to estimate the effective components $\bu_\calJ$ and $\bff$ more accurately as $p$ grows via neural networks \citep{fan2024factor} and the FAST-NN in that paper allows us to nonparametrically select the variable set $\calJ$.  As noted in \cite{fan2024factor}, the FAST model is very flexible and includes most existing models, from factor regression to sparse models in  both parametric and nonparametric forms.  Note that given $\bff$, the FAST model $m(\bff,\bu_{\calJ})$ can also be written as 
$$
m(\bff,\bu_{\calJ})  = \tilde m(\bff,\bx_{\calJ}) = \left \{
\begin{array}{l l}
	\tilde m(\bx_{\calJ})  &  \mbox{specific case I: sparse regression}\\
	\tilde m(\bff) &   \mbox{specific case II: factor regression}
\end{array} \right .
$$ 
for another function $\tilde m$ under the factor model \eqref{eq:factor}.  Therefore, it is a factor-augmented sparse nonparametric regression model.  The model includes the nonparametric sparse regression model $\tilde m(\bx_{\calJ})$  and nonparametric factor regression model  $\tilde m(\bff)$ as two specific examples.  It is applicable to the case where there is no factor structure ($r=0$, covariates are weakly correlated) or low-dimensional setting ($\cJ = $ all variables).

In our current manuscript, we study the application of Deep Learning (DL) methods for estimating {ATE}. Efficient estimation of ATE often involves estimating both the responses (corresponding to treatment and control groups) and the propensity score, the conditional probability of receiving treatment given the covariates \citep{hirano2003efficient}, via the Augmented Inverse Propensity Weighting AIPW. Given an experimental unit with treatment assignment indicator $T$ and covariate $\bx$, its propensity score is defined as
\begin{align}\label{eq:propensity}
	\pi^*(\bx)=\EE\qth{T|\bx} .
\end{align} 
Deep Learning is an extremely useful estimation tool when the structures of the target functions are unknown and possibly nonlinear. We term the strategy of using DL to learn both response and propensity component as the \textbf{Double Deep Learning} ({DDL}) technique, and our proposed estimator of the ATE will combine the benefits of such deep learning strategies. In the current literature, it is unclear whether DL methods are valuable tools for ATE estimation in the presence of strong covariance dependency and sparsity. For applying DL to handle strong covariates dependence, it is sensible to perform a denoising step to capture the independent components of the high-dimensional covariates and use the projected data to perform function estimation. However, the dependency structure is often misspecified (e.g., incorrect knowledge about $r$), leading to incorrect constructions of the denoising algorithms. It is known in the literature that model misspecifications can hurt propensity estimation significantly and lead to biased estimation of the ATE ~\citep{drake1993effects}. From a practitioner's perspective, it is desirable to have efficient {ATE} estimators that can counter the practical issues mentioned above. On the other hand, the ATE estimation strategy should be flexible to tackle the case where the covariates are given to be low-dimensional, and we do not need to estimate the factor structure. In brief, we address the following:\\

{\it {Can using Double Deep Learning for responses and propensity estimation lead to efficient ATE estimation, both in the case of low-dimensional covariates and high-dimensional covariates with or without factor structures, even under model misspecifications?}}\\

In this paper we answer this question affirmatively. We show that for high-dimensional covariates, even when an over-specification $\bar r$ of the factor dimension $r$ is provided (this includes the useful case that covariates are weakly correlated, but the factor model is used.), we can construct consistent factor augmented and deep learning based ATE estimators. Our results allow the covariate dimension to be significantly larger than the sample size, leading to resolving the problem in high dimensions. The versatility of our inference also provides the option to remove the factor modeling component when dealing with low-dimensional scenarios. Our ATE estimator is asymptotically Gaussian and semiparametrically efficient.

\subsection{Our contributions}

\subsubsection*{Methodological contribution}
To our knowledge, our work is the first to introduce factor augmented deep learning techniques in the context of ATE estimation and analyze their theoretical guarantees. We propose a double deep learning type estimator called {\textbf{FIDDLE}}, that stands for the \textbf{F}actor \textbf{I}nformed \textbf{D}ouble \textbf{D}eep \textbf{L}earning \textbf{E}stimator. Suppose we have $n$ observations of the response, treatment indicator, and covariate values, given by $\{(y_i,T_i,\bx_i)\}_{i=1}^n$. Our algorithm consists of three major steps:
\begin{itemize}
	\item \textbf{The pretraining factor augmentation step:} This step aims to estimate the factor components that determine the response and propensity (target) functions.  We introduce a novel diversified projection matrix construction to perform factor augmentation. This is the only step where we use an independent pretraining sample of negligible size, vanishing compared to $n$. When the covariate dimension is low, we remove this factor augmentation step from our method, and the following steps remain the same. 
	
	\item \textbf{The double deep learning step:} We estimate the outcome and propensity functions using factor-augmented deep neural networks. Specifically, we use a newly constructed diversified matrix to construct the {FAST-NN} \citep{fan2024factor} type estimators  $\hat\mu_0,\hat \mu_1,\hat\pi$.
	
	\item \textbf{The ATE estimation step:} We use the structure of the {\it Augmented Inverse Propensity Weighted ({AIPW})} estimator \citep{glynn2010introduction}, to combine the deep-learning-based FAST-NN estimators in the last step, and apply them to the same set of data without any sample splitting to construct the ATE estimator FIDDLE
	\begin{align}\label{eq:aipw}
		\mufidd
		=\frac 1n \sum_{i=1}^n 
		&\left\{\pth{{T_iy_i\over \hat \pi(\bx_i)}-{(1-T_i)y_i\over 1-\hat\pi(\bx_i)}}
		-(T_i-\hat \pi(\bx_i))\pth{{\hat \mu_1(\bx_i)\over \hat \pi(\bx_i)}+{\hat\mu_0(\bx_i)\over 1-\hat \pi(\bx_i)}}\right\}.
	\end{align}
	The doubly robust structure of the AIPW estimator enables us to combine the consistency of the estimators for response and propensity to produce an efficient estimator of {ATE}.
\end{itemize}
	
\begin{remark}[Comparison with previous algorithms]
	The double deep learning-type strategies to combine deep-learning-based estimators are not new in the ATE estimation literature \citep{du2021dimension,farrell2015robust}. However, such constructions often require knowledge of the exact low-dimensional structure of the target functions \cite[Condition 1]{du2021dimension} and often considers other dependency structures \citep{farrell2015robust} and sparsity as in the FAST model, which can significantly degrade the performance. In comparison, the construction of FIDDLE employs factor-augmented deep learning strategies to draw inference under dependency assumptions and learn low-dimensional target functions algorithmically. FIDDLE is also able to work with an overspecified number of factors, which significantly extends its applicability. In addition, if practical knowledge suggests that the data-generating response and propensity functions do not depend on the factor components, we apply our algorithm without the factor-augmentation step, thereby avoiding the pretraining step. Such scenarios often arise when the covariate dimension is small. 
\end{remark}

\subsubsection*{Contributions to theory and applications} Our paper is the first to show asymptotic normality of the {FAST-NN} based AIPW estimator for estimating the average treatment effect. We only use a negligible pretraining sample (compared to the sample size $n$) to construct the diversified projection matrix, and the rest of the deep learning-oriented estimator construction does not involve any further sample splitting. This makes it challenging to provide theoretical guarantees for the corresponding ATE estimator FIDDLE. In particular, our theoretical contributions are threefold:
\begin{itemize}
	\item \textbf{Results for our new diversified projection matrix:} We show that our proposed diversified projection matrix adheres to the requirements in the literature \citep{fan2022learning} and that its singular values are large enough to produce strong estimation guarantees for the response and propensity functions. Our construction differs from the previous method in \cite{fan2024factor} where an incoherence condition on the sample variance matrix \citep{candes2007sparsity, abbe2020entrywise} is required in order to deduce boundedness of the diversified projection matrix (see \prettyref{def:div-proj}). In contrast, our new construction is simple and removes such a requirement.
	
	\item \textbf{Estimation guarantees for response and propensity functions:} We demonstrate that, under the assumption of a hierarchical composition model, factor-augmented neural network estimators can provide optimal guarantees even when the covariate dimension is high. Additionally, these intermediate steps help identify the active components of the covariates in the propensity and response functions, thereby providing interpretable results. This provides valuable information for policy-making, answering questions such as which covariate components influence the assignment of individual units to treatment and control groups, as well as their corresponding outcomes. Our theoretical guarantees for response function estimation using the factor-augmented neural network deviates from the existing work of \cite{fan2024factor}, that studied the function estimation problem with a fixed dataset, as we need to use the random subsamples of control and treatment groups to estimate the response functions $\mu_0^*,\mu_1^*$ respectively. We improve on the above work and provide a detailed analysis of controlling the estimation errors in such random setups.

	\item \textbf{Efficiency guarantees for FIDDLE: } We show that under some broad and relaxed smoothness assumptions on the outcome and propensity functions, FIDDLE enjoys asymptotic normality with semiparametric efficiency for ATE estimation. The analysis comes with significant challenges as we avoid sample splitting to perform the ATE estimation. The semiparametric efficiency is a desirable property in the literature for such tasks \citep{farrell2015robust,fan2022optimal}, as this helps to construct confidence intervals for the unknown treatment effects.  
	
\item \textbf{Contributions in numerical studies:} We also present comparisons of our methods with many classical off-the-shelf ATE estimation techniques and demonstrate how a factor-oriented denoising step helps improve performance in high-dimensional scenarios. Our studies support our theoretical results and show that the accuracy of our estimators increases impressively as the covariate dimensions grow large, even beyond the sample size. The methods we compare against include other regularized neural networks, Generative Adversarial Networks (GANs), and Causal Forest, among others. In terms of studying semi-synthetic data, we use the CIFAR-10 dataset (Canadian Institute For Advanced Research), which demonstrates the excellent performance of FIDDLE over other state-of-the-art methods for ATE estimation, particularly as the dimensionality of covariates $\boldsymbol{x}$ or sample size increases. Furthermore, we apply FIDDLE and benchmark methods to a real-world dataset from the Metabolic and Bariatric Surgery Accreditation and Quality Improvement Program (MBSAQIP) to evaluate the causal effects of different bariatric surgery procedures on weight loss after $30$ days of surgery. 
\end{itemize}

\begin{remark}[Comparison with similar existing results]
	To establish the asymptotic distribution of {FIDDLE}, we apply a proof technique based on the concept of bracketing integral \citep{vaart2023empirical} to control the randomness of the estimators without sample splitting, which was inspired by the work of \cite{fan2022optimal}. Still, the analysis differs significantly as  we work with (a) an AIPW-type estimator instead of the {IPW} and  (b) deep neural networks instead of the standard non-parametric structural assumption. The use of deep neural networks makes our work model agnostic and algorithmic, providing a more general guarantee. In addition, it is known in the literature \citep{glynn2010introduction} that the AIPW estimator obtains ATE estimators with lower variance compared to the {IPW} estimator, which helps our cause as well. We also establish that even in the presence of the factor structure, which leads to high correlation in the covariates, we can achieve the above results. In particular, our guarantees excel when the covariate dimension is significantly large.  The randomness of subsamples
	$\{(y_i,\bx_i)\}_{i\in [n],T_i=1}$ and $\{(y_i,\bx_i)\}_{i\in [n],T_i=0}$ also contribute to the technical proofs.
\end{remark}

\subsection{Related works}
%

Factor models play a crucial role in uncovering low-dimensional latent  structures. Foundational contributions include \citep{chamberlain1982arbitrage} and \citep{bai2003inferential}, which established identification and inference under general factor structures. We learn latent factors based on Diversified Projections (DP) \citep{fan2022learning}, which allows for low-sample size and purposeful overestimation of latent factors for robustness. Our estimator incorporates DP to recover the shared latent structure and improve both treatment and outcome estimation.  

Recent developments in machine learning have further enriched the landscape of causal inference, particularly in high-dimensional or nonlinear regimes. Double Machine Learning (DML) \citep{chernozhukov2018double} formalizes orthogonalization and sample splitting for inference under ML-based function estimation. Causal Forests \citep{wager2018estimation} adapt random forests to estimate conditional average treatment effects (CATE) using specialized split criteria. GANITE \citep{yoon2018ganite} uses generative adversarial networks to learn counterfactual outcomes and derive individualized treatment effects. Recent methodological reviews \citep{hoffmann2024double, brand2023recent} provide comprehensive evaluations of these approaches. Furthermore, recent work on Calibrated Debiased Machine Learning (C-DML) \citep{vanderlaan2024automatic} introduces novel doubly robust estimators that maintain asymptotic linearity even under misspecification. While these methods offer flexibility and strong empirical performance, many do not explicitly account for latent factor structure in the covariates, and several are sensitive to model misspecification due to reliance on either outcome or treatment models alone. In contrast, our method integrates the strengths of factor structural modeling and modern machine learning by combining factor-based learning, neural network estimation, and the AIPW framework into a unified pipeline. 




Deep neural networks (DNNs) \citep{lecun2015deep} have shown state-of-the-art performance in high-dimensional learning tasks and can recover low-dimensional structure \citep{mousavi2015deep,chen2025multilook}. Recent studies \citep{yarotsky2017error, kohler2021rate} provide non-asymptotic guarantees across function classes. In nonparametric regression, DNNs mitigate the curse of dimensionality \citep{bauer2019deep, schmidt2020nonparametric, fan2024noise, bhattacharya2024deep} with the property of adaptively and algorithmically learning low-dimensional structure. Our method leverages the FAST-NN architecture proposed by \cite{fan2024factor}, which is designed to adaptively capture sparse and dense components of the covariate space. This flexibility makes it particularly suitable for estimating both propensity scores and outcomes in AIPW estimation. Furthermore, \cite{farrell2021deep} offers theoretical support for using DNNs in semiparametric estimation without sample splitting, aligning with our unified approach for efficient estimation of {ATE}.

\subsection{Organization of the manuscript}

The remainder of the paper is organized as follows. 
\prettyref{sec:prep} introduces the setup, notation, and structural definitions of the problem that underpin our framework. 
\prettyref{sec:method} introduces our proposed estimator FIDDLE, and its components. \prettyref{sec:theory} formally describes the model assumptions and the theoretical guarantees for our estimator, including consistency and convergence rates. \prettyref{sec:numerical} presents simulation studies that compare the performance of our method with existing benchmarks under various simulated and real datasets. Additional supporting results are provided in the appendix.

\section{Preparation}
\label{sec:prep}

We build our model using a fully connected deep neural network with ReLU activation $\bar \sigma(\cdot)=\max\sth{\cdot,0}$ similar to \cite{fan2024factor}. Before presenting our methodology and results, we provide definitions that we will rely on throughout the manuscript.

\begin{definition}[Deep \relu Networks]
	Let $L$ be any positive integer and
	$\bd = (d_1, . . . , d_{L+1}) \in \naturals^{L+1}$. A deep ReLU network $g:\reals^{d_0} \to \reals^{d_{L+1}}$ is given as
	the form
	\begin{align}
		\label{eq:relu}
		g(\bx) = \calL_{L+1} \circ \bar\sigma \circ \calL_{L} \circ \bar\sigma \circ \cdots \circ \calL_2 \circ \bar\sigma \circ \calL_1(\bx),
	\end{align}
	where $\calL_{\ell}(\bz) = \bW_{\ell}\bz + \bb_\ell$ is a linear transformation with the weight parameters $\bW_{\ell} \in \reals^{d_{\ell}\times d_{\ell-1}},\bb_\ell \in \reals^{d_{\ell}}$, and $\bar\sigma : \reals^{d_{\ell}}\mapsto  \reals^{d_\ell}$ applies the ReLU activation function coordinatewise.
	
\end{definition}

\begin{definition}[Deep \relu network class]\label{def:relu-class}
	For any $L\in \naturals,\bd\in \naturals^{L+1}, B,M\in \reals^+\cup \{\infty\}$, the deep \relu network family $\calG(L,\bd,M,B)$ with truncation level $M$, depth $L$, width vector $\bd$, and weight bound $B$ is given as
	\begin{align*}
		\calG(L,\bd,M,B)=\sth{\Tr_M(g(\bx)): g \text{ of form \eqref{eq:relu} with } \|\bW_{\ell}\|_{\max}\leq B,\|\bb_{\ell}\|_{\max}\leq B},
	\end{align*}
	where $\Tr_{M}(\cdot)$ is the coordinatewise truncation operator given by $[\Tr_M(\bz)]_i=\sgn(z_i)(|z_i|\wedge M)$ and $\|\cdot\|_{\max}$ denotes the suppremum norm of a vector. The class of deep \relu networks with depth $L$ and width $N$ is given by the specific case $\bd=(d_{in},N,N,\dots,N,d_{out})$, and we denote it throughout the text by $\calG(L,d_{in},d_{out}, N,M,B)$.
\end{definition}
We will use the following class of hierarchical composition functions to model $\mu_0^*,\mu_1^*,\pi^*$.
\begin{definition} [$(\beta,C)$-smooth functions]  A $d$-variate function $f$ is called $(\beta,C)$-smooth for $\beta,C>0$ if the following is satisfied. Decompose $\beta$ into integer part $r\geq 0$ and fraction part $0<s<1$. Then given every non-negative sequence $\alpha \in \naturals^d$ with $\sum_{j=1}^d \alpha_j = r$, the partial derivative $(\partial f)/(\partial x_1^{\alpha_1}
	\dots x_d^{\alpha_d})$ exists, and
	$\left|{\partial^r f\over
			\partial x_1^{\alpha_1}
			\dots \partial x_d^{\alpha_d}}
		(\bx) -
		{\partial^rf\over
			\partial x_1^{\alpha_1}
			\dots \partial x_d^{\alpha_d}}
		(\bz)\right|
		\leq C \|\bx - \bz\|^s_2. 
	$
\end{definition}

\begin{definition}[Hierarchical composition of smooth functions \citep{kohler2021rate,fan2024factor}]
\label{def:hierarchical}
Fix a constant $C>0$. Let $\calH(d,l,\calP)$ denote the class of $l$-depth and $d$-variate hierarchical composition of $(\beta,C)$-smooth functions for $(\beta,t)$ in a set $\calP$ with
$\label{eq:hierarchical}
		\calP \subset [1,\infty)\times \naturals^+,  \sup_{(\beta,t)\in \calP} (\beta \vee t) <\infty$
	\begin{itemize}
		\item $(l=1)$ We have the set of all $t$-variate functions with $(\beta,C)$ smoothness
		\begin{align*}
			\calH(d, 1,\calP) = &\left\{h : \reals^d \mapsto \reals : h(\bx) = g(\bx_{\calJ}), \text{ where } g : \reals^t \mapsto \reals \text{ is }\right.
			\nonumber\\  
			&~ \left.\text{$(\beta,C)$-smooth for some $(\beta, t) \in \calP$ and $\calJ\in [d],|\calJ|=t$}\right\}
		\end{align*}
		\item $(l\geq 2)$ We recursively define $\calH(d,l,\calP)$ as
		\begin{align*}
			\calH(d, l,\calP) = &\left\{h : \reals^d \mapsto \reals : h(\bx) = g(f_1(\bx),\dots, f_t(\bx)), \text{ where } g : \reals^t \mapsto \reals \text{ is }  \right.
			\nonumber\\  
			&~ \left.\text{$(\beta,C)$-smooth for some $(\beta, t) \in \calP$ and $f_i\in \calH(d,l-1,\calP),i\in [t]$}\right\}
		\end{align*}
	\end{itemize}
\end{definition}

Basically, $\calH(d,l,\calP)$ consists of the $l$ time compositions of $t$-variate functions of $(\beta, C)$ smoothness for any $(t, \beta) \in \cP$.
The accuracy of estimating $\mu_0^*,\mu_1^*,\pi^*\in \calH(d, l,\calP)$ will be quantified by the parameter $\gamma^*$ indicating the hardness of the above composition class.

\begin{definition}[Hardness parameter of $\calH(d, l,\calP)$]\label{def:dim-adjusted-smoothness}
	Given any $\calP$ satisfying \eqref{def:hierarchical} the hardness quantifier $\gamma^*$ of the worst case error of approximating any function in $\calH(d, l,\calP)$ by a deep \relu network is quantified by $\gamma^* ={\beta^* \over d^*} \text{ with } (\beta^*, d^*) = \argmin_{(\beta,t)\in \calP}{\beta\over t}$.
	In view of \cite{kohler2021rate}, we restrict to the case where all the compositions has a smoothness parameter
	$\beta \geq  1$ to simplify the presentation. The parameter $\gamma^*$ originates from the following approximation result of \cite{fan2024factor} (Theorem 4 therein), in which $\beta/t$ reflects the dimension-adjusted degree of smoothness in a component of the hierachical composition model.
\end{definition}

\begin{lemma}[Approximating $\calH(d, l,\calP)$ via deep \relu Networks]
	\label{thm:approx-smooth}
	Let $g$ be a $d$-variate, $(\beta, C)$-smooth function. There exists some universal constants $c_1$--$c_5$ depending only on $d, \beta, C$, such that for arbitrary $N\in \mathbb{N}^+ \setminus \{1\}$, there exists a deep \relu network $g^\dagger \in \mathcal{G}(c_1, d, 1, c_2 N, \infty, c_3 N^{c_4})$ satisfying $
		\|g^\dagger-g\|_{\infty,[0,1]^d} \le c_5 N^{-2\beta/d}$.
	Furthermore, if $g\in \mathcal{H}(d, l,\mathcal{P})$ with $\sup_{(\beta, t)\in \mathcal{P}} (\beta \lor t) <\infty$ and $g$ is supported on $[-c_6, c_6]^d$ for some constant $c_6$. There also exists some universal constants $c_{7}$--$c_{11}$ such that for arbitrary $N\in \mathbb{N}^+ \setminus \{1\}$, there exists a deep \relu network $g^\dagger \in \mathcal{G}(c_7, d, 1, c_8 N, \infty, c_9 N^{c_{10}})$ satisfying
	$\|g^\dagger-g\|_{\infty,[-c_{6},c_6]^d} \le c_{11} N^{-2\gamma^*}$.
\end{lemma}

\begin{definition}[Bracketing number and integral \citep{vaart2023empirical}]\label{def:bracketing}
	Given any distribution $P$, a function class $\calF$ and a fraction $\epsilon>0$, let $\calN_{[]}(\epsilon,\calF,\|\cdot\|)$ denote the $\epsilon$-bracketing number of $\calF$ under any norm $\|\cdot\|$, i.e., the minimum number of $\epsilon$-brackets needed to cover $\calF$ in the $\|\cdot\|$ distance.
	Denote the bracketing integral as 
	$
	\tilde J_{[]}\pth{\delta,\calF,\|\cdot\|}
	= \int_0^\delta \sqrt{1+\log \calN_{[]}(\epsilon,\calF,\|\cdot\|)}\ d\epsilon.
	$
	We will pick a suitable norm later to fit our analysis.
\end{definition}

\begin{definition}[Diversified projection (DP) matrix \citep{fan2022learning,fan2024factor}]\label{def:div-proj}
	Let $\bar r \geq r$ and $c_1$ be a universal positive constant. A $p \times \bar r$ matrix $\bW$ is called a DP matrix if it satisfies (a) Boundedness: $\|\bW\|_{\max} \leq c_1$, (b) Exogeneity: $\bW$ is independent of $\bx_1,\dots , \bx_n$, (c) Significance: the matrix $\bH = p^{-1} \bW^\top \bB \in \reals^{\bar r\times r}$ satisfies $\nu_{\min}(\bH) \gg p^{-1/2}$.
	Each column of $\bW$ is called a diversified weight, and $\bar r$ is the number of diversified weights.
\end{definition}

\section{Methodology: \fidd}
\label{sec:method}

Our proposed estimator FIDDLE  is a {\it double deep learning estimator} that relies on estimating both the outcome and propensity function using factor-augmented sparse throughput neural networks (FAST-NN) and then applying the AIPW estimator \eqref{eq:aipw}. To obtain the above deep learning-based estimators, we use the idea of the FAST estimator introduced in \cite{fan2024factor} that uses a LASSO \citep{tibshirani1996regression} type penalized loss function. We describe the estimator below. Let $\bW\in \reals^{\barr\times p}$ be a given diversified projection matrix as defined in \prettyref{def:div-proj} (a construction of $\bW$ used in our work is outlined below). Suppose that we have the data $\{(y_i,T_i,\bx_i)\}_{i=1}^n$. Then estimate the factor component of $\bx_i$ as
\begin{align}
	\tilde \bff_i=\frac 1p \bW^\top \bx_i,\quad i=1,\dots,n.
	\label{eq:factors_daffodil}
\end{align}
To describe our objective functions to construct the deep learning estimators, define the clipped-$L_1$ function $\psi_\tau (x)$ with the clipping threshold $\tau > 0$ as
$
\psi_\tau (x) = {|x|\over \tau} \wedge 1.
$
Define $n_0=\sum_{i=1}^n (1-T_i),\quad n_1=\sum_{i=1}^n T_i$. Then the penalized mean squared error objectives $\hat R_{0},\hat R_{1},\hat R_{2}$ corresponding to estimating $\mu_0^*, \mu_1^*,\pi^*$ are defined as (the choice of the tuning parameters $\lambda_0,\lambda_1, \lambda_2, \tau_0,\tau_1, \tau_2, B,M,\bar r$ to guarantee our results will be described later)
\begin{align}
	\begin{gathered}
		\hat R_{t}(g, \bTheta)=\frac 1{n_t}\sum_{i=1,T_i=t}^n \sth{y_i - g
			\pth{\qth{\tilde \bff_i,\Tr_M(\bTheta^\top \bx_i)}}}^2 
		+\lambda_t\sum_{i,j}\psi_{\tau_t}(\Theta_{i,j}),\quad t=0,1
		\\
		\hat R_{2}(g, \bTheta)=\frac 1n\sum_{i=1}^n \sth{T_i - g
			\pth{\qth{\tilde \bff_i,\Tr_M(\bTheta^\top \bx_i)}}}^2 
		+\lambda_2\sum_{i,j}\psi_{\tau_2}(\Theta_{i,j})
	\end{gathered}
	\label{eq:l1-optim}
\end{align}
where $[x, y]$ denotes the concatenation of two vectors $x\in \reals^{d_1}$ and $y \in \reals^{d_2}$ to form a $(d_1 + d_2)$-dimensional vector, $\Tr_{M}(\cdot)$ is the truncation operator defined in \prettyref{def:relu-class}.  Following \cite{fan2024factor}, we optimize the above loss functions over $g\in \calG(L,\bar r+N,1,N,M,B)$, the \relu deep network class given via \prettyref{def:relu-class}, and $\bTheta\in \reals^{p\times N}$. Given any estimators $\hat g,\hat\bTheta$ originating from the above optimization, denote the corresponding FAST-NN estimator as
\begin{align}
	\label{eq:mfast}
	m^{\fast}(\bx; \bW, \hat g, \hat \bTheta) =  \hat g
	\pth{\qth{\tilde \bff,\Tr_M(\hat \bTheta^\top \bx)}}.
\end{align} 
In light of the above, we are now ready to present our primary estimators.

\subsection{Constructing a diversified projection matrix}
To construct a diversified projection matrix $\bW$ we first randomly pick $\{i_1,\dots,i_m\}\subset[n]$ and consider the spectral decomposition of the corresponding variance covariance matrix $\frac 1m \sum_{j=1}^m\bx_{i_j}\bx_{i_j}^\top$ to obtain the eigenvalues
$\{\hat \lambda_j\}$ and eigenvectors $\{\hat \bv\}$ so that
$$
\frac 1m \sum_{j=1}^m\bx_{i_j}\bx_{i_j}^\top
=\sum_{j=1}^m \hat \lambda_j \hat \bv_j\hat \bv_j^\top,\quad
\lambda_1\geq\lambda_2\geq\dots,\lambda_m\geq 0.
$$
Then we propose the following  novel construction of a diversified projection matrix
\begin{align}
	\label{eq:div-proj}
	\bW = \Biggl[\sqrt{\hat \blambda_1}\cdot \hat \bv_1,\dots,\sqrt{\hat \blambda_{\barr}}\cdot \hat \bv_\barr\Biggr]
\end{align}
We will show later in \prettyref{thm:div-proj} that $\bW$ satisfies the requirements of a diversified projection matrix with a constant-order smallest singular value. For showing theoretical guarantees, we can use $m=n^{1-\gamma}$ for some constant $\gamma>0$ and use $\{(y_i,T_i,\bx_i):i\in[n]/\{i_1,\dots,i_m\}\}$ for ATE estimation.  Therefore, the pretraining sample size $m$ is negligible.  For the convenience of notations, we will assume from this point onward an access to a $\bW$ that is independent of the data, whose size is indexed by $n$. As $m$ is negligible with respect to $n$, our theoretical results presented later will remain the same in view of the construction of $\bW$ above.

\begin{remark}[Comparison with the previous construction of DP matrix]
	\cite{fan2024factor} uses the matrix $\tilde \bW=\sqrt{p}[\hat \bv_1,\dots \hat \bv_\barr]$ as their choice of the DP matrix. Notably, showing the boundedness requirement for $\tilde \bW$ as in \prettyref{def:div-proj} is challenging, and requires the incoherence assumption in \cite{abbe2020entrywise}. For example, it is challenging to satisfy the boundedness requirement of \prettyref{def:div-proj} for the submatrix $\sqrt p [\hat\bv_{r+1},\dots,\hat\bv_{\barr}]$ of $\tilde\bW$, as the usual argument based on Weyl's Theorem \cite[Lemma 2.2]{chen2021spectral} provides significantly weaker controls on the magnitudes of eigenvalues $\hat\lambda_{r+1},
	\dots,\hat\lambda_{\barr}$ when the data generating $\bB$ matrix in \eqref{eq:factor} is of rank $r<\barr$. Our modifications for constructing $\bW$ directly guarantee the boundedness requirements and provide a more natural candidate for the DP matrix compared to $\tilde \bW$.
\end{remark}

\subsection{Response function estimation}
To estimate the outcome functions corresponding to the control and treatment groups, we run two separate FAST-NN on the data $\{(y_i,\bx_i):T_i=0,i\in [n]\}$ and $\{(y_i,\bx_i):T_i=1,i\in [n]\}$ respectively, and define the FAST-NN estimators for estimating $\mu_0^*,\mu_1^*$ as
\begin{align}
		\hat g_i(\cdot), \hat \bTheta_i \in \argmin_{
		\bTheta\in\reals^{p\times N}\atop g\in \calG(L,\barr+N,1,N,M,B)}\hat R_i(g,\bTheta),
	\quad 
		\hat \mu_{i}^{\fast}(\cdot) = m^{\fast}(\cdot; \bW, \hat g_i, \hat \bTheta_i),
		\quad 
		i=0,1.
		\label{eq:mu-fast}
\end{align}

\subsection{Propensity function estimation}
To estimate the propensity function $\pi^*$ given in \prettyref{eq:propensity} we construct the FAST-NN estimator using the treatment indicators for the experimental units:
\begin{align}\label{eq:fast-propensity}
	\begin{gathered}
		\hat g_2(\cdot), \hat \bTheta_2 \in \argmin_{
			\bTheta_2\in\reals^{p\times N}\atop g_2\in \calG(L,\barr+N,1,N,M,B)}\hat R_2(g,\bTheta),
		\quad \tilde \pi(\cdot) = m^{\fast}(\cdot; \bW, \hat g_2, \hat \bTheta_2).
	\end{gathered}
\end{align}
Note that to aid the theoretical analysis later on, we do not initially impose any restrictions to ensure that $m^{\fast}(\cdot; \bW, \hat g_2, \hat \bTheta_2)$ will lie within $[0,1]$, which is the range for the true propensity score. We perform a subsequent truncation step to obtain the final propensity estimator $\hat\pi^{\fast}=\max\sth{\alpha_n,\min\sth{\tilde \pi^\fast,1-\alpha_n}}$ for a suitable $\alpha_n\in [0,1] $ to be chosen.

\subsection{ATE estimation}
Ultimately, the {\it double deep learning}-type ATE estimator FIDDLE is given as
\begin{align}\label{eq:aipw-fast}
	\mufidd
	=\frac 1n \sum_{i=1}^n 
	&\left\{\pth{{T_iy_i\over \hat \pi^{\fast}(\bx_i)}-{(1-T_i)y_i\over 1-\hat\pi^{\fast}(\bx_i)}}\right.
	\nonumber\\
	&\quad \left.-(T_i-\hat \pi^{\fast}(\bx_i))\pth{{\hat \mu_1^{\fast}(\bx_i)\over \hat \pi^{\fast}(\bx_i)} + {\hat\mu_0^{\fast}(\bx_i)\over 1-\hat \pi^{\fast}(\bx_i)}}\right\}.
\end{align}

	\begin{remark}[Modifying our algorithm for low-dimensional covariates]
		When the covariates $\bx_1,\dots,\bx_n$ have a low dimension, the factor augmentation step becomes redundant. In that case, we modify our algorithm by replacing $\tilde \bff_i$ in \eqref{eq:factors_daffodil} with $\bx_i$ for all $i\in \{1,\dots,n\}$ and set $\bTheta = 0$ in \eqref{eq:l1-optim}. For simplicity of presentation, the reference to the FIDDLE method will also include such modifications. The proof of the theoretical results presented later accommodates this specific scenario in the case $r=0$. 
		
	\end{remark}

	\section{Theory}
	\label{sec:theory}
	
	\subsection{Model}
	
	We assume that data $\sth{(y_i,T_i,\bx_i)}_{i=1}^n$ are independently and identically distributed realizations of random variables $(y,T,\bx)$. 
	The model that generates $(y,T,\bx)$ is given by
	\begin{align}\label{eq:model-output}
		y(t) = \mu_t^*(\bx) +\varepsilon(t), t\in \{0,1\},\quad \PP\qth{T=1|\bx} = 1- \PP\qth{T=0|\bx} = \pi^*(\bx),
	\end{align}
	where $\varepsilon(0),\varepsilon(1)$ are mean zero random variables. The goal is to estimate $\mu = \EE\qth{\mu_1^*(\bx)-\mu_0^*(\bx)}.$
	We assume the factor model $\bx=\bB\bff + \bu$ as in \eqref{eq:factor}, and model the functions $\mu_0^*,\mu_1^*,\pi^*$ as 
	\begin{align}\label{eq:low-dim-functions}
		\mu_0^*(\bx)=\mu_0^*(\bff,\bu_{\calJ_0}),
		\mu_1^*(\bx)=\mu_1^*(\bff,\bu_{\calJ_1}),
		\pi^*(\bx)=\pi^*(\bff,\bu_{\calJ_2}), \calJ_0,\calJ_1,\calJ_2\subset\sth{1,\dots,p}.
	\end{align}

	\subsection{Assumptions}
	\label{sec:asmp}
	
	\begin{myassumption}[Low dimensionality]
		\label{asmp:low-dim}
		$r,|\calJ_0|,|\calJ_1|,|\calJ_2|$ are at most finite constants.
	\end{myassumption}
	
	\begin{myassumption}[Response and propensity bounds]
		\label{asmp:reg}
		$\|\mu_0^*\|_\infty,\|\mu_1^*\|_\infty \leq M^*,
		\pi^*\in (\alpha_*,1-\alpha_*)$
		for constants $M^*,\alpha_*\in (0,1)$ and $\mu_0^*,\mu_1^*,\pi^*$ are $c_1$-Lipschitz for some constants $c_1>0$. We further assume that $1\leq M^* \leq M\leq c_2 M^*$ for some constant $c_2 > 1$, where $M$ is the trimming parameter used in constructing the {FAST-NN} estimators in \prettyref{sec:method}.
	\end{myassumption}
	
	\begin{myassumption}[Unconfoundedness]\label{asmp:confoundedness}
		$T$ is independent of $(y(0),y(1))$ given $\bx$.
	\end{myassumption}
	
	\begin{myassumption}[Boundedness]
		\label{asmp:bounded}
		For the factor model (2.2), there exist universal constants $c_1$ and $b$ such that (a) the factor loading matrix satisfies $\|\bB\|_{\max} \leq c_1$, (b) the factor component $\bff$ of $\bx$ is zero-mean and supported on $[-b, b]^r$, and (c) the idiosyncratic component $\bu$ of $\bx$ is zero-mean and supported on $[-b, b]^p$.
		This also implies that covariates are bounded in each coordinate and $\bx_1,\dots,\bx_n\in [-K,K]^p$ for some constant $K>0$. 
	\end{myassumption}
	
	\begin{myassumption}[Weak dependence]
		\label{asmp:depend}
		$\sum_{j,k\in [p],j\neq k} \abs{\EE[u_ju_k]}\leq c_1\cdot p$ for some constant $c_1$.
	\end{myassumption}
	
	\begin{myassumption}[Sub-Gaussian noise]
		\label{asmp:noise}
		There exists a universal constant $c_1$ such that 
		$$\PP\qth{\abs{\varepsilon(0)} \geq t|\bff,\bu},\PP\qth{\abs{\varepsilon(1)} \geq t|\bff,\bu} \leq 2e^{-c_1t^2}$$
		for all the $t > 0$ almost surely.
	\end{myassumption}
	
	\begin{myassumption}[Pervasiveness]
		\label{asmp:pervasiveness}
		${p\over c_1}< \lambda_{\min}(\bB^\top \bB) \le \lambda_{\max}(\bB^\top \bB) \le c_1 p$ for a constant $c_1$.
	\end{myassumption}
	
	\begin{myassumption}[Weak dependence between $\bff$ and $\bu$]
		\label{asmp:dist-fu}
		$\|\bB \bSigma_{\bff, \bu}\|_{F} \le c_1 \sqrt{p}$ for a constant $c_1>0$, where $\bSigma_{\bff, \bu} = \mathbb{E} [\bff \bu^T] \in \mathbb{R}^{r\times p}$ is the covariance matrix between $\bff$ and $\bu$. 
	\end{myassumption}
	
	\begin{myassumption}[Function Class and ReLU Hyperparameters]
		\label{asmp:hyperparameters}
		The true response and propensity functions satisfy
		$\mu_0^*\in \calH(r+|\calJ_0|,l,\calP),
		\mu_1^*\in \calH(r+|\calJ_1|,l,\calP),
		\pi^*\in \calH(r+|\calJ_2|,l,\calP)$
		for some bounded constants $r, l, |\calJ_0|,|\calJ_1|,|\calJ_2|$, and $\calP$ has the dimension-adjusted smoothness $\gamma^*>\frac 12+c_0$  for a constant $c_0>0$, where $\gamma^*$ is given by \eqref{def:dim-adjusted-smoothness}). The following conditions on the deep ReLU network hyperparameters hold for constants $c_1,\dots,c_6$ which only depend on $l$ and $\mathcal{P}$ of $\{\mathcal{H}(r+|\calJ_j|, l,\mathcal{P})\}_{j=0,1,2}$.
		\begin{align}
			\begin{gathered}
				c_1 \le L \leq c_2,\quad  c_3 \log n \le \log B \leq c_4\log n,
				\quad 
				r\le \overline{r} \lesssim c_3\\
			c_5 (n/\log n)^{\frac{1}{4\gamma^*+2}} \le N \leq c_6 (n/\log n)^{\frac{1}{4\gamma^*+2}}.
			\end{gathered}
		\end{align}
		
	\end{myassumption}
	
	\begin{remark}[Discussion of the assumptions]
		\prettyref{asmp:low-dim} and \prettyref{asmp:reg} are standard in the deep learning literature \citep{kohler2021rate,fan2024noise}. \prettyref{asmp:confoundedness} is also standard in the Causal Inference literature \citep{hirano2003efficient}, which ensures that there are no unmeasured confounders. \prettyref{asmp:bounded} through \prettyref{asmp:pervasiveness} are also borrowed from the factor modeling literature \citep{fan2024factor}. \prettyref{asmp:dist-fu} subsumes the standard assumption of independence of $\bff$ and $\bu$ in the factor modeling literature, which is usually needed for identifiability of the model \citep{fan2021robust}. \prettyref{asmp:hyperparameters} provides necessary constraints on the complexity of the outcome and propensity models to guarantee asymptotic normality of {FIDDLE}. This is also standard in the literature of nonparametric regressions via deep neural networks for achieving optimal mean squared errors \citep{fan2024factor} and the class of functions is indeed very broad, including additive models or more generally the compositions of low-dimensional functions.
	\end{remark}
	
	\subsection{Main results}
	We begin with a guarantee for our proposed diversified projection matrix, as shown in \prettyref{eq:div-proj}.
	\begin{theorem}
		\label{thm:div-proj}
		The diversified projection matrix $\bW$ constructed in \eqref{eq:div-proj} is a valid diversified projection under \prettyref{asmp:bounded} through
		\prettyref{asmp:dist-fu}. In addition, there exist universal constants $c_1,c_2,c_3$ independent of $m, p, t, r, \overline{r}$ such that
		\begin{align*}
			c_1 - c_2 \left(r\sqrt{\frac{\log p + t}{m}} + r^2\sqrt{\frac{\log r + t}{m}} + \frac{1}{\sqrt{p}}\right)\le \nu_{\min}(p^{-1}{\bW}^\top \bB) \le \nu_{\max}(p^{-1}{\bW}^\top \bB) \le c_3 .
		\end{align*} 
	\end{theorem}
	
	Note that due to our slight modification of the construction of $\bW$, we do not require incoherent type of conditions.  We now present function estimation guarantees. For $\hat R_0,\hat R_1,\hat R_2$ in \eqref{eq:l1-optim} and optimization error $\delta_{\opt}$, define $\{(\hat g_t,\hat\bTheta_t)\}_{t=0}^2$ as
	\begin{align}
		\hat R_{t}\pth{\hat g_t,\hat \Theta_t}
		&\leq \inf_{
			\bTheta\in\reals^{p\times N}\atop g\in \calG(L,\barr+N,1,N,M,B)}\hat R_{t}\pth{g, \bTheta}
		+\delta_{\opt},
		\quad t=0,1,2.
		\label{eq:mu1-loss}
	\end{align}
	Consider the FAST-NN estimators $\hat \mu_0^\fast,\hat \mu_1^\fast,\hat \pi^\fast$ defined as
	\begin{align}\label{eq:estimators}
		\begin{gathered}
			\hat \mu_0^\fast(\bx)=m^\fast(\bx;\bW,\hat g_0,\hat\bTheta_0),
			\quad \hat \mu_1^\fast(\bx)=m^\fast(\bx;\bW,\hat g_1,\hat\bTheta_1),\\
			\hat \pi^\fast(\bx)=\max\sth{1/\log n,\min\sth{ m^\fast(\bx;\bW,\hat g_2,\hat\bTheta_2),1-1/\log n}},
		\end{gathered}
	\end{align}
	Given any function $h$ and $j=0,1$, define $\|h\|_{n,j}^2= \frac 1{n_j}\sum_{i:T_i=j} h^2(\bx_i),\|h\|_{n}^2= \frac 1{n}\sum_{i=1}^n h^2(\bx_i)$, and $\|h\|_2^2 = \int h^2(\bx)dP(\bx)$, where $P$ is the law of $\bx$. Then we have the following result.
	\begin{theorem}[Oracle-type inequality for FAST-NN estimator]
		\label{thm:main-fast-nn}
		Suppose that all the assumptions in \prettyref{sec:asmp} hold, except for \prettyref{asmp:hyperparameters}, which is not used in this context. Consider the FAST model \eqref{eq:low-dim-functions} and the FAST-NN estimator obtained by solving \eqref{eq:mu-fast} and \eqref{eq:fast-propensity} with $N,B$ large enough such that $N\geq 2(r + \max\{|\calJ_j|:j=0,1,2\}), B \geq c_1r \max\{|\calJ_j|\}_{j=0}^2$,
		\begin{align*}
			&\lambda_j \geq c_2 {\log(n_jp(N + r)) + L\log(BN)\over n_j},\quad 
			&&\tau_j^{-1} \geq c_3(r + 1)(BN)^{L+1}(N + \bar r)p{n_j},\quad j=0,1\\
			&\lambda_2 \geq c_2 {\log(np(N + r)) + L\log(BN)\over n},\quad 
			&&\tau_2^{-1} \geq c_3(r + 1)(BN)^{L+1}(N + \bar r)p{n}
		\end{align*}
		for some constants $c_1,c_2,c_3$, and the number of diversified projections $\bar r \geq r$. Define 
		$$
		\begin{gathered}
			\delta_{i,a} = \inf_{g\in\calG(L,\barr+N,1,N,M,B)}\|g - \mu_i^*\|_\infty^2,\quad i=0,1,\quad
			\delta_{2,a} = \inf_{g\in\calG(L,\barr+N,1,N,M,B)}\|g - \pi^*\|_\infty^2,\\
			\delta_{i,s} = {(N^2L +N\bar r){L\log(BNn)}\over n} + \lambda_i|\calJ_i|,
			\quad
			\delta_{i,f} = {|\calJ_i|r \cdot \bar r\over {\nu_{\min}^2(\bH) \cdot p}}, \quad i=0,1,2.
		\end{gathered}
		$$ 
		Then, with probability at
		least $1 - 3e^{-t}$, the following holds, for $n$ large enough,
		\begin{align}
			\| \hat \mu_{j}^\fast - \mu_j^*\|_2^2+\| \hat \mu_{j}^\fast - \mu_j^*\|_{n,j}^2
			&\leq {c_4\over \alpha_*^2}\sth{\delta_\opt + \delta_{j,a} + \delta_{j,s} + \delta_{j,f} +\frac tn},\quad j=0,1 \label{fastnn1}\\
			\| \hat \pi^{\fast} - \pi^*\|_2^2+\| \hat \pi^{\fast} - \pi^*\|_n^2
			&\leq {c_4}\sth{\delta_\opt + \delta_{2,a} + \delta_{2,s} + \delta_{2,f} +\frac tn}, \label{fastnn2}
		\end{align}
		where $c_4$ is a constant and $\alpha_*$ is as given in \prettyref{asmp:bounded}.
	\end{theorem}
	
	The results \eqref{fastnn1},\eqref{fastnn2} give the mean squared errors for estimating the outcome functions and propensity score function in terms optimization error, approximation error,  complexity of neural networks, penalization biases, as well as the estimation error of latent factors.  The proof of the above results do not follow from the standard estimation guarantees for FAST-NN with fixed datasets as in \cite{fan2024factor}, as our estimators $\hat \mu_0^\fast$ and $\hat\mu_1^\fast$ are constructed using the random sub-samples corresponding to the treatment group and control groups.  The mean squared errors are measured with respected to the probability measure of the covariate $\bx$ and its empirical version. 
	
	Next, we present asymptotic and efficiency guarantees of $\mufidd$ for estimating $\mu$.
	
	\begin{theorem}[Asymptotic normality of {FIDDLE}]
		\label{thm:main-ate}
		Assume that all the assumptions in \prettyref{sec:asmp} hold and that
		$(n/ \log n)^{\frac 12 +c_1} < p <  n^{100}$ for some $c_1\in (0,\frac 12)$ depending on $c_0$ in  \prettyref{asmp:hyperparameters}.
		Let $\hat\mu_0=\hat\mu_0^\fast, \hat\mu_1=\hat\mu_1^\fast, \hat \pi=\hat \pi^\fast$ be as in \eqref{eq:estimators} with $\delta_\opt < (n/ \log n)^{-{\gamma^*\over 2\gamma^*+1}}$ and tuning parameters 
		\begin{align*}
			&\lambda_j = c_2 {\log(n_jp(N + r)) + L\log(BN)\over n_j},\quad 
			&&\tau_j^{-1} = c_3(r + 1)(BN)^{L+1}(N + \bar r)p{n_j}n,\quad j=0,1,\\
			&\lambda_2 = c_2 {\log(np(N + r)) + L\log(BN)\over n},\quad 
			&&\tau_2^{-1} = c_3(r + 1)(BN)^{L+1}(N + \bar r)p{n^2},
		\end{align*}
		where $c_2,c_3$ are constants as in \prettyref{thm:main-ate}. Then the ATE estimator FIDDLE \eqref{eq:aipw-fast} satisfies
		\begin{eqnarray*}
		&&\sqrt n (\mufidd -\mu)\to N(0,\sigma^2),\\
		&& \sigma^2=\EE\left[ (\mu_1^*(\bx) - \mu_0^*(\bx)-\mu)^2  + {\Var[y(1)|\bx] \over \pi^*(\bx)}+ {\Var[y(0)|\bx] \over 1-\pi^*(\bx)} \right],
		\end{eqnarray*}
		and $\sigma^2$ attains the semiparametric efficiency bound \cite[Theorem 1]{hahn1998role}. In addition, if $r=0$, i.e., $\bx$ does not have a factor component, 
		then the result holds for any $p\leq n^{100}$. 
	\end{theorem}

	\begin{remark}[Discussion of the results]

		\prettyref{thm:main-ate} establishes asymptotic normality and semiparametric efficiency of FIDDLE even when there is strong dependency among the covariates. We require only that the dimensionality-adjusted degree of smoothness $\gamma^*$ satisfies $\gamma^* > 1/2$  in \prettyref{asmp:hyperparameters}. Additionally, our proof shows that if $\mu_0^*,\mu_1^*,\pi^*$ have different dimensionality-adjusted degree of smoothness $\gamma_0^*,\gamma_1^*, \gamma_2^*$ then we can establish the above result by only requiring $\gamma_0^*\gamma_2^*>\frac 14$ and $\gamma_1^*\gamma_2^*>\frac 14$, rather than having $\gamma_i^*>\frac 12, i=0,1,2$. This establishes the doubly-robustness of {FIDDLE}, which relaxes the complexities of outcome and propensity models. Under the factor model assumptions, the requirement of a large covariate dimension $p$ is essential to consistently estimating the factor components, see, e.g., \cite[Theorem 3]{fan2024factor}. The additional requirement of $p>(n/ \log n)^{\frac 12 +c_1}$ is imposed to guarantee the stronger result of asymptotic normality of the AIPW estimator, which can be removed in the absence of the factor component. This is captured in the second half of \prettyref{thm:main-ate} with the case $r=0$. In addition, the asymptotic normality and semiparametric efficiency hold even when our algorithm uses an overspecified number of factors $\barr$. Our proof involves applying empirical process theory to establish a uniform error bound over the set of possible estimators within complex neural network classes. 
	\end{remark}

	\section{Numerical studies}
	\label{sec:numerical}

	\subsection{Candidate methods}
	
	\label{ss:method-compare}
	We will compare the performance of FIDDLE with the following alternative approaches for estimating the average treatment effect. See \prettyref{app:sim-imp} for implementation details. 
	
	\begin{itemize}
		\item \textbf{Vanilla Neural Networks (Vanilla-NN):} A baseline variant that replaces FAST-NN with the trained fully connected neural networks, and with everything else the same.
	
	\item \textbf{Generative Adversarial Nets for Individualized Treatment Effects (GANITE):} A GAN-based approach \citep{yoon2018ganite} that first generates counterfactual outcomes via a dedicated generator, followed by training an individualized treatment effect (ITE) estimator. The ATE is estimated by the sample average of the estimated ITEs.
	
	\item \textbf{Double Robust Forest (DR):} A forest-based implementation of the doubly robust estimator \citep{bang2005doubly}. It jointly estimates the propensity scores and outcome models using random forests and computes the ATE via the AIPW estimator.
	
	\item \textbf{Double Machine Learning Forest (DML):} A double machine learning framework of \cite{chernozhukov2018double}, which employs cross-fitting procedures to estimate nuisance parameters and eliminate regularization bias. It utilizes forest learners for both propensity score and response estimation, then applies double machine learning for ATE estimation. 
	
	\item \textbf{Causal Forest (CF):} A nonparametric method based on generalized random forests \citep{wager2018estimation, athey2019generalized}. It estimates the responses by ensembling classification and regression trees (CART), and computes the ATE by their weighted difference.
	
	\item \textbf{Causal Forest on Latent Factors (Factor-CF):} A variant of the Causal Forest method applied exclusively to the latent factors $\boldsymbol{f}$ extracted from the covariates $\boldsymbol{x}$. 
	
	\item \textbf{Oracle Inverse Propensity Weighting (Oracle-IPW):} Oracle benchmark using the ground truth propensity scores and corresponding IPW estimator \cite{robins2000marginal}.
	
	\item \textbf{Oracle Augmented Inverse Propensity Weighting (Oracle-AIPW):} Oracle benchmark using the true response and propensity functions for AIPW \citep{robins1994estimation}.

	\end{itemize}
	
	\subsection{Analysis with simulated data}
	\label{sec:simulated-data}
	
	We conduct two Monte Carlo experiments using synthetic datasets to evaluate the empirical performances. The first experiment benchmarks FIDDLE against a range of state-of-the-art machine learning estimators for the {ATE}, focusing on performance across varying covariate dimensions. The second experiment investigates how FIDDLE behaves as the sample size increases. In both experiments, the data-generating process incorporates latent factor structures and nonlinear treatment and outcome models, enabling us to assess the robustness of estimators in high-dimensional, complex settings. Each experiment is replicated 100 times. Denote  $\sigma(x)= \frac{e^x}{1 + e^x}, \mathrm{trun} \{z\} = 0.8 z + 0.1$ for the rest of the paper.
	
	\paragraph{Data Generating Process.}
	We assume that the covariate vector $\boldsymbol{x}$ follows a factor model with a loading matrix $\boldsymbol{B} = (b_{ij})_{p \times r} \in \mathbb{R}^{p \times r}$, a vector of latent factors $\boldsymbol{f} = (f_i)_{r} \in \mathbb{R}^{r}$, and an idiosyncratic component $\boldsymbol{u} = (u_i)_{p} \in \mathbb{R}^{p}$, with $b_{ij} \overset{\mathrm{i.i.d}}{\sim} \mathrm{Unif}(-\sqrt{3}, \sqrt{3}),\quad 
	f_i, u_i \overset{\mathrm{i.i.d}}{\sim} \mathrm{Unif}(-1,1)$. The number of factors is fixed at $r = 4$, and we evaluate performance for $p \in \{10, 50, 100, 500, 1000, 5000, 10000\}$. The sample size is set to $n = 5000$. The propensity and response models presented below incorporate nonlinear interactions of $\bff$ and $\bu$. We model $\pi^*(\bx)$ as $\pi^*(\boldsymbol{f}, \boldsymbol{u})
	= \mathrm{trun} \{ \sigma(\sin(f_1) + \tan(f_2) + f_3 + f_4 + \sum_{j=1}^{5} u_j), \}$
	the outcome $y$ as
	$y= \mu^*(\boldsymbol{f}, \boldsymbol{u}) + T \, \tau^*(\boldsymbol{f}, \boldsymbol{u_{\mathcal{J}}}) + \varepsilon$, with $\varepsilon \sim \mathcal{N}(0, 1/4)$ independent of $\boldsymbol{f}, \boldsymbol{u}$ and
	\begin{align*}
		\mu^*(\boldsymbol{f}, \boldsymbol{u}) &= 10 + f_1 + f_2 f_3 + \sin(f_4) + \log(5 + u_1 + u_2 u_3) + \tan(u_4) + u_5\\
		\tau^*(\boldsymbol{f}, \boldsymbol{u}) &= 5 + f_1 + f_2 + \sin(f_3) + \tan(f_4) + u_1 + u_2 + \sin(u_3 + u_4) + \tan(u_5)
		.
	\end{align*}
	The ground truth ATE with the specified model below is $\mathbb{E} [ \tau^*(\boldsymbol{f}, \boldsymbol{u}) ] = 5$.

	\begin{figure}[t]
		\centering
		\includegraphics[width=\linewidth]{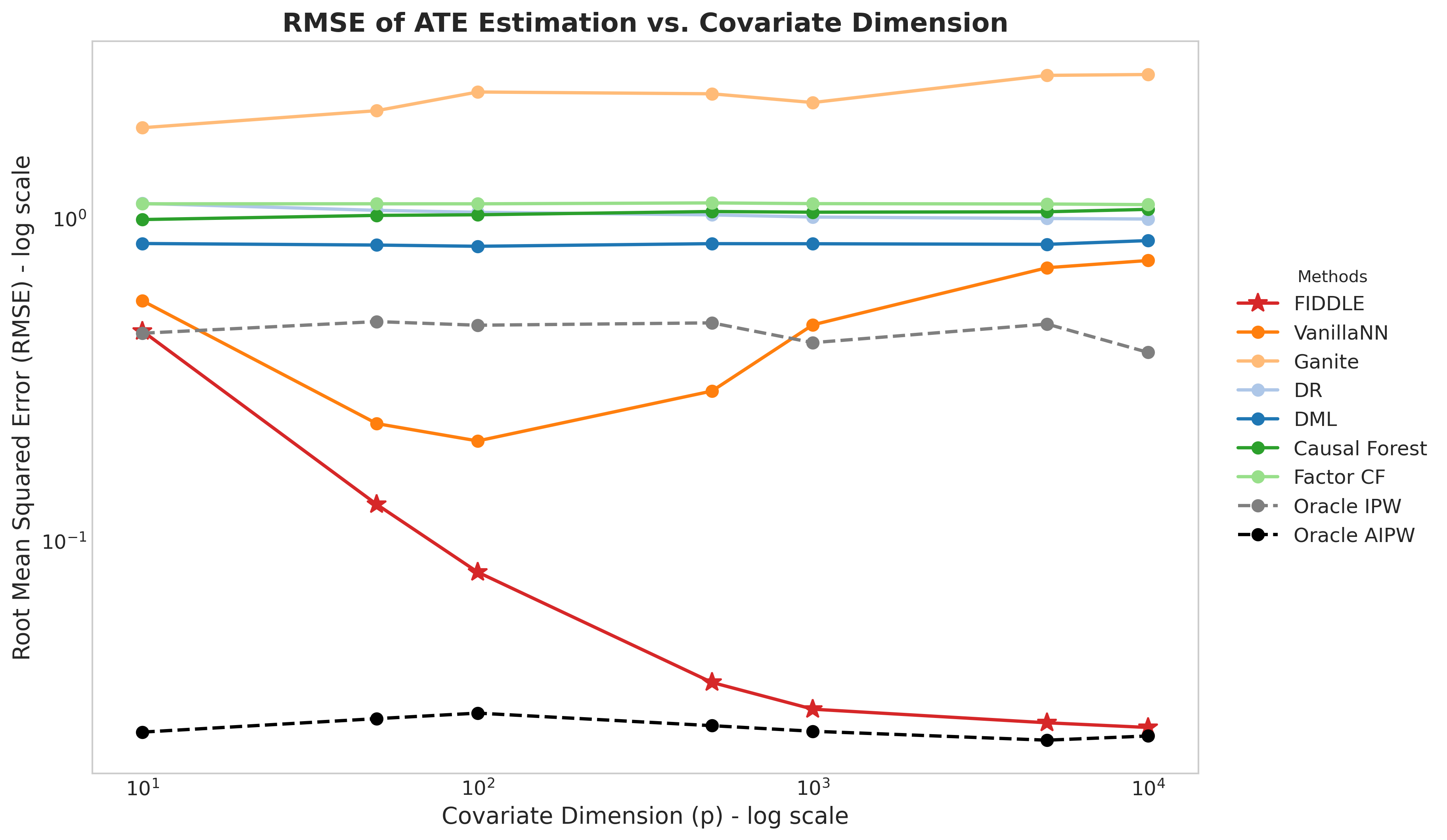}
		\caption{RMSE for different numbers of covariates by candidate ATE estimation methods.}
		\label{fig:method_compare}
	\end{figure}

	\paragraph{Results with varying covariate dimensions.} \prettyref{fig:method_compare} shows that the proposed estimator FIDDLE demonstrates a remarkable advantage over all other non-oracle methods in root mean squared error (RMSE), achieving consistently superior performance (see \prettyref{tab:method_compare} for the exact values and standard errors (SE)). This is due to improved accuracy of FIDDLE in estimating latent factors and highlights our method's robustness and ability to leverage latent factor structure. Remarkably, when $p$ is sufficiently large, FIDDLE attains performance comparable to that of the oracle-AIPW estimator. These empirical patterns align with the theoretical guarantees of {FIDDLE}. Moreover, FIDDLE achieves performance comparable to the oracle-IPW estimator when the covariate dimension is relatively low and surpasses it as $p$ increases. In comparison, competing estimators—whether factor-based or neural network-based—fail to realize similar gains, with no appreciable decline in RMSE as $p$ grows, underscoring their limitations in handling complex, high-dimensional covariate structures and nonparametric function learning.
	
	\begin{table}[t]
		\centering
		\begin{tabular}{l|ccccccc}
			\hline
			\textbf{Model} & $p=10$ & $p=50$ & $p=100$ & $p=500$ & $p=1000$ & $p=5000$ & $p=10000$ \\
			\hline
			Oracle AIPW & 0.0255 & 0.0280 & 0.0292 & 0.0267 & 0.0256 & 0.0240 & 0.0248 \\
			& (0.0021) & (0.0020) & (0.0020) & (0.0018) & (0.0019) & (0.0017) & (0.0016) \\
			\hline
			FIDDLE & 0.4467 & 0.1295 & 0.0799 & 0.0363 & 0.0300 & 0.0272 & 0.0263 \\
			& (0.0258) & (0.0076) & (0.0055) & (0.0027) & (0.0021) & (0.0020) & (0.0019) \\
			\hline
			VanillaNN & 0.5548 & 0.2308 & 0.2039 & 0.2912 & 0.4673 & 0.7028 & 0.7399 \\
			& (0.0248) & (0.0076) & (0.0069) & (0.0090) & (0.0108) & (0.0148) & (0.0154) \\
			\hline
			Ganite & 1.9107 & 2.1582 & 2.4663 & 2.4348 & 2.2871 & 2.7776 & 2.7944 \\
			& (0.0176) & (0.0204) & (0.0174) & (0.0427) & (0.0355) & (0.0858) & (0.0765) \\
			\hline
			DR & 1.1123 & 1.0588 & 1.0441 & 1.0252 & 1.0095 & 0.9990 & 0.9953 \\
			& (0.0148) & (0.0078) & (0.0081) & (0.0061) & (0.0078) & (0.0063) & (0.0061) \\
			\hline
			DML & 0.8355 & 0.8263 & 0.8191 & 0.8346 & 0.8340 & 0.8305 & 0.8532 \\
			& (0.0198) & (0.0129) & (0.0127) & (0.0103) & (0.0112) & (0.0114) & (0.0090) \\
			\hline
			CF & 0.9914 & 1.0214 & 1.0258 & 1.0501 & 1.0449 & 1.0482 & 1.0663 \\
			& (0.0186) & (0.0128) & (0.0112) & (0.0096) & (0.0121) & (0.0108) & (0.0090) \\
			\hline
			Factor-CF & 1.1094 & 1.1096 & 1.1094 & 1.1165 & 1.1108 & 1.1080 & 1.1030 \\
			& (0.0049) & (0.0054) & (0.0050) & (0.0049) & (0.0059) & (0.0046) & (0.0052) \\
			\hline
			Oracle IPW & 0.4401 & 0.4783 & 0.4660 & 0.4739 & 0.4114 & 0.4699 & 0.3838 \\
			& (0.0260) & (0.0300) & (0.0325) & (0.0353) & (0.0280) & (0.0276) & (0.0275) \\
			\hline
		\end{tabular}
		\caption{RMSE and standard error (in parentheses) of candidate ATE estimation methods across different covariate dimensions $p$ with fixed sample size $n=5000$ and $100$ replications.}
		\label{tab:method_compare}
	\end{table}
	
	\paragraph{Results for different sample sizes}
	\label{ss:size-compare}
	
	We run the experiments with the previous values of $p$ and $n \in \{1000, 2000, 3000, 4000, 5000, 6000, 7000, 8000, 9000, 10000\}$. As shown in \prettyref{fig:size_compare_heatmap} (see  \prettyref{tab:size_compare} for the standard errors (SE) and \prettyref{fig:size_compare} for a $\log-\log$ plot), FIDDLE demonstrates a pronounced and consistent reduction in both root mean square error (RMSE) as $n$ increases, across all covariate dimensions $p$. This trend is particularly evident in high dimensions, where traditional methods often suffer from instability or bias. The rapid convergence of FIDDLE with increasing $n$ highlights its statistical efficiency, affirming that the method scales well even with large dimensionality of the covariates.

	\begin{figure}[t]
		\centering
		\includegraphics[width=\textwidth]{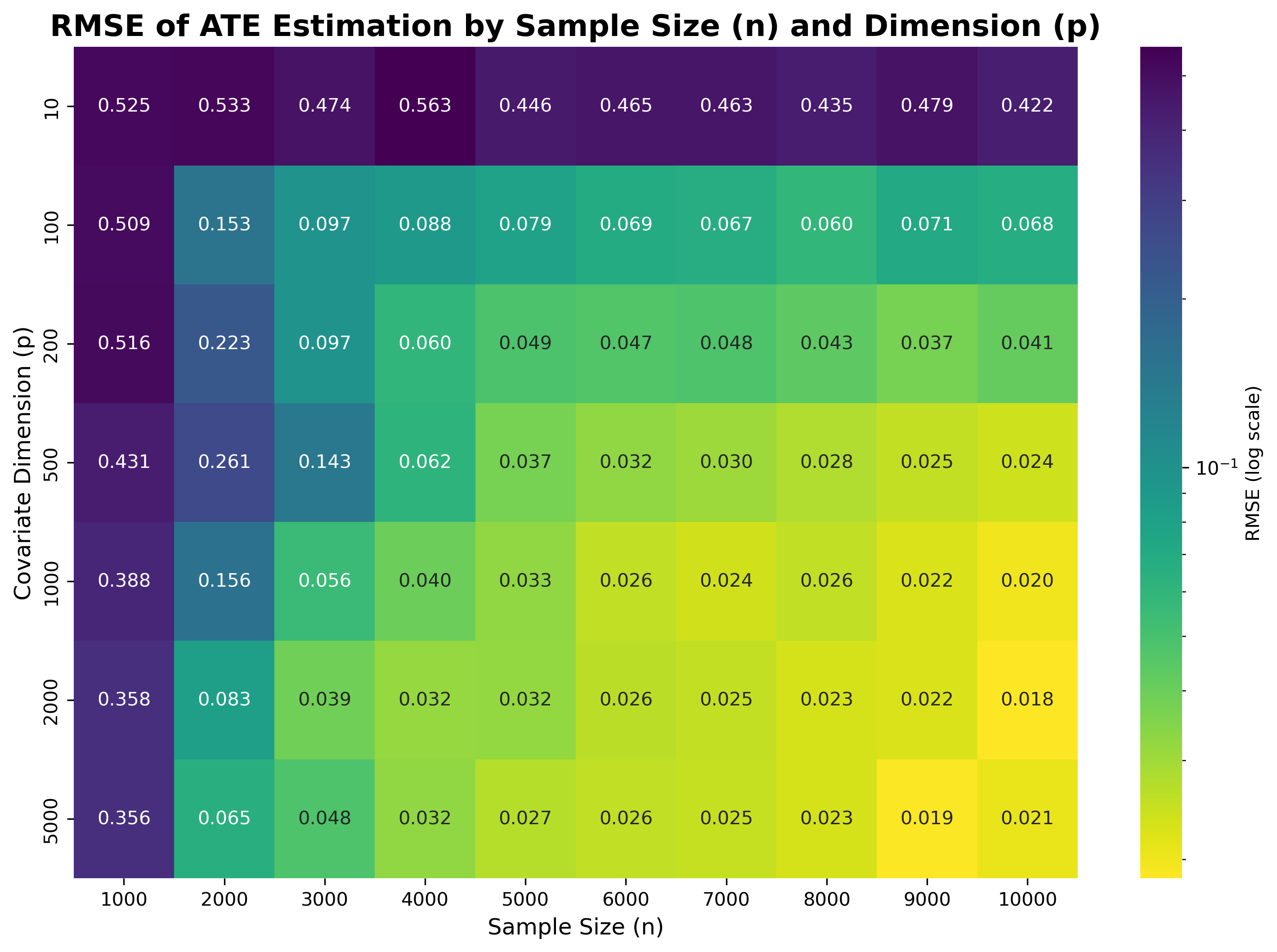}
		\caption{RMSE of FIDDLE for ATE estimation across different sample sizes and covariate dimensions. Dark regions and light regions indicate high and low RMSE, respectively.}
		\label{fig:size_compare_heatmap}
	\end{figure}

	\subsection{Application to semi-synthetic image data}
	\label{ss:application-image}
	
	In this section, we demonstrate the practical utility of FIDDLE through a semi-synthetic image-based simulation derived from the Canadian Institute for Advanced Research (CIFAR-10) dataset. The CIFAR-10 dataset (Canadian Institute For Advanced Research) \citep{krizhevsky2009learning} is a widely used benchmark in machine learning and computer vision. It consists of 60,000 color images of size $32 \times 32$ pixels, categorized into 10 different classes. To facilitate analysis, we reshape the multidimensional array into a two-dimensional matrix $\boldsymbol{X}$, where $n = 60,000$ observations and $p = 32 \times 32 \times 3 = 3,072$ covariates. The covariate matrix $\boldsymbol{X} \in \mathbb{R}^{n \times p}$ is normalized and decomposed into its factor structure $
	\boldsymbol{X} = \boldsymbol{F} \boldsymbol{B}^T + \boldsymbol{U}$,
	where $\boldsymbol{B} = (\boldsymbol{b}_1, \dots, \boldsymbol{b}_p)^T \in \mathbb{R}^{p \times r}$ represents the loading matrix, $\boldsymbol{F} = (\boldsymbol{f}_1, \dots, \boldsymbol{f}_n)^T \in \mathbb{R}^{n \times r}$ denotes the latent factors, and $\boldsymbol{U} = (\boldsymbol{u}_1, \dots, \boldsymbol{u}_p)^T \in \mathbb{R}^{n \times p}$ is the residual component with the number of factors $r = 4$. The unknown factors are estimated via a least-squares optimization algorithm: $
	\underset{\boldsymbol{B} \in \mathbb{R}^{p \times r}, \boldsymbol{F} \in \mathbb{R}^{n \times r}}{\mathrm{minimize}} \sum_{i = 1}^n \lVert \boldsymbol{X}_i - \boldsymbol{B} \boldsymbol{f}_i \rVert^2 = \lVert \boldsymbol{X} - \boldsymbol{F} \boldsymbol{B}^T \rVert^2_F$.
	We randomly sample $n' = 5,000$ observations from the entire dataset as $\boldsymbol{x}$ to replicate the semi-synthetic data-generating process. The factor $\Tilde{\boldsymbol{f}}$ is selected from the solution $\Tilde{\boldsymbol{F}}$, corresponding to $\boldsymbol{x}$, and the residuals are computed as $\Tilde{\boldsymbol{u}} = \boldsymbol{x} - \Tilde{\boldsymbol{f}} \Tilde{\boldsymbol{B}}^T$. 
	The treatment $T$ is generated by Bernoulli sampling with probability $\pi^*(\bx)$, where $\pi^*(\bx)$ is modeled as $\pi^*(\widetilde{\boldsymbol{f}}, \widetilde{\boldsymbol{u}})
	= \mathrm{trun} \{ \sigma(\sin(\widetilde{f}_1) + \sum_{i = 2}^4 \widetilde{f}_i + \sin (\widetilde{u}_1) + \sum_{j = 2}^5 \widetilde{u}_j)\}$
	with $\sigma(\cdot),\mathrm{trun}(\cdot)$ are as in \prettyref{sec:simulated-data}.
	We model the outcomes as $y= \mu^*(\widetilde{\boldsymbol{f}}, \widetilde{\boldsymbol{u}}) + T \, \tau^*(\widetilde{\boldsymbol{f}}, \widetilde{\boldsymbol{u}}) + \varepsilon$, where
	\begin{align*}
		\mu^*(\widetilde{\boldsymbol{f}}, \widetilde{\boldsymbol{u}}) &= 10 + \widetilde{f}_1 + \sin(\widetilde{f}_2) + \widetilde{f}_3 \widetilde{f}_4 + \widetilde{u}_1 (\widetilde{u}_2 + \sin(\widetilde{u}_3)) + \widetilde{u}_4 + \widetilde{u}_5, \\
		\tau^*(\widetilde{\boldsymbol{f}}, \widetilde{\boldsymbol{u}}) &= \widetilde{f}_1 (\widetilde{f}_2 + 3) + \widetilde{f}_3 + \sin(\widetilde{f}_4) + \sin(\widetilde{u}_1) + \widetilde{u}_2 + \widetilde{u}_3 \widetilde{u}_4 \widetilde{u}_5,  \\
		\varepsilon & \overset{\mathrm{i.i.d}}{\sim} \mathcal{N}(0, 1/4), \quad \text{independent of } \widetilde{\boldsymbol{f}}, \widetilde{\boldsymbol{u}}.
	\end{align*}
	
	The ground truth ATE is empirically estimated for each simulation, and all candidate methods are implemented identically to Section \ref{ss:method-compare}.

	\textbf{Results.} Table~\ref{tab:CIFAR-10} reports the root mean squared error (RMSE) and standard error (SE) of each candidate method on the semi-synthetic CIFAR-10 dataset over 100 replications. FIDDLE achieves the best performance among all non-Oracle methods. The remarkable closeness of FIDDLE's performance to Oracle-AIPW supports our semiparametric efficiency claim. In contrast, the Vanilla Neural Network (Vanilla-NN) has a substantially higher MSE, and other methods perform considerably worse. FIDDLE also outperforms the Oracle-IPW estimator, highlighting the added stability gained through the doubly robust AIPW framework.
	
	\begin{table}[H]
		\centering
		\small
		\begin{tabular}{c c c c c c c c c}
			\hline
			{$\text{Oracle}\atop \text{AIPW}$} & \footnotesize{\fidd} & $\text{Vanilla} \atop \text{NN}$ & \footnotesize{GANITE} & \footnotesize{DR} & \footnotesize{DML} & \footnotesize{CF} & $\text{Factor}\atop\text{CF}$ & {$\text{Oracle}\atop\text{IPW}$} \\
			\hline
			0.009 & 0.030 & 0.282 & 1.389 & 1.664 & 1.427 & 1.878 & 1.990 & 0.448 \\
			(0.001) & (0.003) & (0.012) & (0.032) & (0.007) & (0.007) & (0.007) & (0.006) & (0.030) \\
			\hline
		\end{tabular}
		\caption{RMSE and its standard error (in parentheses) of candidate ATE estimation methods on the semi-synthetic dataset based on CIFAR-10 over $100$ replications. } 
		\label{tab:CIFAR-10}
	\end{table}
	
	\subsection{Application with real dataset from bariatric surgery}
	\label{ss:application-surgery}
	
	We evaluate the causal effect of different bariatric surgery procedures on short-term weight loss using different candidate ATE estimators. Bariatric surgery remains the most effective treatment for morbid obesity, achieved by reducing stomach size or altering nutrient absorption pathways.
	Our analysis is based on the 2017 Participant Use File (PUF) from the Metabolic and Bariatric Surgery Accreditation and Quality Improvement Program (MBSAQIP) \citep{facs_mbsaqip}, which collects high-quality, nationwide data on bariatric surgeries. After preprocessing, the dataset includes $174,013$ patient records with $42$ pretreatment covariates, surgery type, and 30-day BMI reduction as the outcome. We select Sleeve Gastrectomy (Sleeve)—the most widely performed procedure—as the control, and compare it against four common alternatives as the treatment: Roux-en-Y Gastric Bypass (RYGB), Adjustable Gastric Band (AGB), Biliopancreatic Diversion with Duodenal Switch (BPD/DS), and Single Anastomosis Duodeno-Ileal Bypass with Sleeve Gastrectomy (SADI-S). Table~\ref{tab:surgery} reports the estimated ATE on 30-day BMI reduction, and associated 95\% confidence intervals over $100$ replications.
	
	As shown in Table~\ref{tab:surgery}, FIDDLE yields relatively robust results that support the clinical understanding of the mechanism of each surgical procedure and the expected impact on short-term weight loss. RYGB shows a small negative {ATE}, consistent with equivalent short-term outcomes \citep{arterburn2018comparative}. AGB demonstrates a substantially negative {ATE}, reflecting its restrictive mechanism that produces slower weight loss requiring behavioral adaptation \citep{hady2012impact}. BPD/DS yields a positive {ATE}, indicating a superior early reduction in BMI through its combined restrictive-malabsorptive approach \citep{hutter2013outcome}. SADI-S shows a moderately negative {ATE}, providing the first quantitative evidence that its staged design prioritizes long-term metabolic benefits over immediate weight loss enhancement \citep{pereira2024beyond}. Taken together, FIDDLE provides ATE estimates that support existing knowledge in the medical community-- less aggressive procedures produce smaller short-term benefits, while more invasive techniques correspond to greater reductions in BMI.

	\begin{table}[t]
		\centering
		\begin{tabular}{c | c c c c}
			\hline
			Surgery & RYGB & AGB & BPD/DS & SADI-S \\
			\hline
			\fidd & -0.0111 & -1.0391 &  0.3364 & -0.5020\\
			~ & (-0.0116, -0.0106) & (-1.0516, -1.0266) & (0.3109, 0.3620) & (-0.5187, -0.4854) \\
			\hline
			Vanilla NN & -0.3814 & -1.6807 & -0.3670 & -1.5438\\
			~ & (-0.3880, -0.3749) & (-1.6913, -1.6701) & (-0.3860, -0.3479) & (-1.5573, -1.5304) \\
			\hline
			GANITE & -0.4511 & -2.4651& -1.5018 & -2.0452\\
			~ & (-0.5161, -0.3861) & (-2.6590, -2.2712) & (-1.6552, -1.3484) & (-2.1686, -1.9218) \\
			\hline
			DR & -0.0310 & -0.7343 &  0.0307 & -0.5118\\
			~ & (-0.0313, -0.0308) & (-0.7412, -0.7273) & (0.0305,  0.0309) & (-0.5178, -0.5058) \\
			\hline
			DML & -0.0154 & -0.2403 &  0.6306 & -0.5457\\
			~ & (-0.0160, -0.0148) & (-0.2725, -0.2080) & (0.6060,  0.6552) & (-0.5498, -0.5416) \\
			\hline
			CF & -0.0225 & -1.0804 &  0.2135 & -0.6469\\
			~ & (-0.0229, -0.0222) & (-1.0815, -1.0792) & ( 0.2114,  0.2157) & (-0.6488, -0.6451) \\
			\hline
			Factor-CF & -0.0569 & -1.0924 &  0.1749 & -0.7773\\
			~ & (-0.0572, -0.0565) & (-1.0936, -1.0912) & (0.1730,  0.1769) & (-0.7787, -0.7759)\\
			\hline
		\end{tabular}
		\caption{Estimated ATE and 95\% confidence intervals for 30-day BMI reduction by different procedures, compared with Sleeve as the control.}
		\label{tab:surgery}
	\end{table}

    \paragraph{Data availability statement:} The data that support the findings of this study is provided in the 2017 Participant Use File (PUF) from the website of Metabolic and Bariatric Surgery Accreditation and Quality Improvement Program (MBSAQIP) and the file is available upon request at their website \citep{facs_mbsaqip}.




\appendix

\section{Proof of \prettyref{thm:div-proj}}

Define $\hat \bSigma =\frac 1m\sum_{i=1}^m \bx_i\bx_i^\top$ and consider the spectral decomposition of $\hat\bSigma$
\begin{align*}
	\hat\bSigma = \sum_{i=1}^p\hat\lambda_i\cdot \hat\bv_i\hat \bv_i^\top.
\end{align*}
Using the coordinatewise boundedness of $\{\bx_{i}\}_{i=1}^m$ we get that each coordinate of $\hat\bSigma$ is bounded. Then we use
$$
\hat \bSigma_{jj}=\sum_{i=1}^p\hat\lambda_i (\hat v_{ij})^2, \quad j=1,\dots p,
$$
which implies $\sqrt{\hat \lambda_i}\cdot \hat \bv_{i}$ has bounded coordinates for all $i=1,\dots,p$. Hence, to show that $\bW=\qth{\sqrt{\hat \lambda_1}\cdot \hat \bv_{1},\dots,\sqrt{\hat \lambda_\barr}\cdot \hat \bv_{\barr}}$ is a valid diversified projection matrix, it sufficies to ensure that it is independent of the data that we project using $\bW$ and that the smallest singular value of $\frac 1p\bW^\top \bB$ is large enough. As we use data splitting to construct the diversified projection matrix, and then use $\bW$ to project the second half of the data, independence comes for free. Hence, it is only left to prove the singular value bounds mentioned in the theorem statement.

To this end, using Weyl's theorem \cite[Lemma 2.2]{chen2021spectral} it follows that
\begin{align}\label{eq:m1}
	|\hat \lambda_i - \lambda_i(\bB\bB^\top)|
	\leq \|\hat\bSigma - \bB\bB^\top\|_F,\quad i=1,\dots, p,
\end{align}
where $\|\cdot\|_F$ denotes the Frobenius norm. In view of \cite[Lemma 5]{fan2024factor} we note that under \prettyref{asmp:bounded}, \prettyref{asmp:depend}, \prettyref{asmp:dist-fu} 
\begin{align}\label{eq:m5}
	\|\hat\bSigma - \bB\bB^\top\|_F
	\leq c_1 p\left(r\sqrt{\frac{\log p + t}{m}} + r^2\sqrt{\frac{\log r + t}{m}} + \frac{1}{\sqrt{p}}\right)
\end{align}
for a universal constant $c_1$. Note that from \prettyref{asmp:pervasiveness} we have that 
$$
\lambda_i(\bB\bB^\top) \in (\frac p{c_2} , c_2 p),\quad i=1,\dots,r
$$
for a large constant $c_2$. Hence, for $m\geq c_3\log p$ with a large constant $c_3>0$ we combine \eqref{eq:m1} and \eqref{eq:m5} with the last display to get
$$
\hat \lambda_{i}\in (\frac p{c_4} , c_4 p),\quad i=1,\dots,r
$$
for a constant $c_4>0$. Next, denoting $\bW_r = \qth{\sqrt{\hat \lambda_1}\hat \bv_1,\dots, \sqrt{\hat \lambda_1}\hat \bv_r},\hat \bV_r=[\hat \bv_1,\dots,\hat \bv_r]$ we get
\begin{align*}
	\nu_{\min}(p^{-1} \bW^\top \bB)\geq \nu_{\min}(p^{-1} \bW_r^\top \bB)
	\geq  \frac 1{\sqrt c_4} \nu_{\min}(p^{-1/2}\hat\bV_r\bB),
\end{align*}
where the above inequalities followed using the Courant-Fischer minimax characterization of the smallest singular value
$\nu_{\min}(\bA)=\min_{\dim(\calU)=1}\max_{\bx\in \calU:\|\bx\|_2=1} \|\bA\bx\|_2$ \cite[Theorem 1]{dax2013eigenvalues}. The rest of the proof follows the strategy of \cite[Proposition 1, Equation (F.23)]{fan2024factor}, and is omitted here.

\section{Proof of \prettyref{thm:main-ate}}

For the section below, we use the following notation. Let $P$ denote the law of the covariate $\bx$. Given any function $h=h(\bx)$ define 
\begin{align*}
	\EE_P[h]=\int h \ dP(\bx), \quad \|h\|_{L_2(P)} = \sqrt{\EE_P[h^2]},
	\quad
	\EE_n[h] = \frac 1n \sum_{i=1}^n h(\bx_i).
\end{align*}
We first provide a general result on the asymptotic normality of the AIPW estimator, given response and propensity estimators $\hat\mu_0,\hat\mu_1,\hat\pi$
\begin{align}\label{eq:aipw2}
	\hat\mu^\aipw
	=\frac 1n \sum_{i=1}^n 
	&\left\{\pth{{T_iy_i\over \hat \pi(\bx_i)}-{(1-T_i)y_i\over 1-\hat\pi(\bx_i)}}
	-(T_i-\hat \pi(\bx_i))\pth{{\hat \mu_1(\bx_i)\over \hat \pi(\bx_i)}+{\hat\mu_0(\bx_i)\over 1-\hat \pi(\bx_i)}}\right\}.
\end{align} 
The following result provides conditions on the complexity of possible function classes for $\hat\mu_0,\hat\mu_1,\hat\pi$, that guarantees the asymptotic normality of $\hat\mu^\aipw$. The proof of \prettyref{thm:main-ate} will rely on verifying these assumptions for the neural network classes of the FAST estimators. 
\begin{theorem}\label{thm:genr}
	Consider estimators $\hat \mu_0, \hat \mu_1,\hat \pi\in \calF$ from function class $\calF$, that are constructed from $\{(y_i,T_i,\bx_i)\}_{i=1}^n$. Suppose that with probability $1-\xi$, the estimators satisfy
	\begin{align}\label{eq:delta-bound}
		\EE_P[(\hat \mu_0 - \mu_0^*)^2]\leq \delta_0^2,
		\quad \EE_P\qth{(\hat \mu_1 - \mu_1^*)^2}\leq \delta_1^2,
		\quad 
		\EE_P\qth{(\hat \pi - \pi^*)^2} \leq \delta_2^2,
	\end{align}
	for the true response and propensity functions $\mu_0^*,\mu_1^*,\pi^*$, where the expectation is taken with respect to new $\bx$ given the data. Assume that $T_i,\pi_i$ are independent of $y_i$, conditional on $\bx_i$. In addition, suppose that the following also holds true
	\begin{enumerate}[label=(\roman*)]
		
		\item $\sup_{f\in \calF}\|f\|_\infty\leq \tilde B$ for some constant $\tilde B>0$,
		
		\item $\alpha<\pi^*,\hat \pi<1-\alpha$ for some $\alpha\in (0,1/2)$,
		
		\item ${\sqrt n\delta_1\delta_2\over \alpha^2} + {\sqrt n\delta_0\delta_2\over \alpha^2} + \sqrt n \xi/\alpha \to 0$, as $n\to \infty$,
		
		\item $\lim_{n\to \infty} {(\tilde J_{[]}({\delta},\calF,L_2(P)))^2\over  \indc{\alpha\delta^2\sqrt n < 1} \alpha\delta^2 \sqrt n +\indc{\alpha\delta^2\sqrt n \geq 1}} = 0$ for $\delta \in \sth{{\delta_0\over \alpha},{\delta_1\over \alpha},{\delta_2\over \alpha^2}}$.
		
	\end{enumerate}
	Then we have that $\sqrt n \pth{\hat \mu^\aipw-\mu} \to N(0,\sigma^2)$, where $$\sigma^2=\EE\left[ (\mu_1^*(\bx) - \mu_0^*(\bx)-\mu)^2  + {\Var[y(1)|\bx] \over \pi^*(\bx)}+ {\Var[y(0)|\bx] \over 1-\pi^*(\bx)} \right].$$
	
\end{theorem}

\begin{remark}
	The result also outlines the {\it doubly-robust} property of the AIPW estimator. Condition (iii) above shows that even when the estimation guarantees of the response functions are poor, i.e., $\delta_0,\delta_1$ converges to zero slowly, we can still guarantee the above asymptotic result as long as we have strong estimation guarantees for the propensity score, meaning $\delta_2\delta_0$ and $\delta_2\delta_1$ converges to zero at a rate $n^{-(1+c)}$ for some constant $c>0$. In particular, if the target functions $\mu_0^*,\mu_1^*,\pi^*$ belong to the classes $\calF_0,\calF_1,\calF_2$ with different dimensionality-adjusted degree of smoothness parameters $\gamma_0^*,\gamma_1^*,\gamma_2^*$ (as in \prettyref{def:dim-adjusted-smoothness}) respectively, and the estimators $\hat\mu_0,\hat \mu_1,\hat \pi$ achieves the nonparametric rates $\delta_i^2 = n^{-{2\gamma_i^*\over 2\gamma_i^*+1}}, i=0,1,2$, then we can establish the above result by requiring $\gamma_i^*\gamma_2^*>\frac 14+c$, rather than individually requiring $\gamma_i^*>\frac 12 + c, i=0,1,2$,  for some $c>0$. The is essentially the doubly robustness property in terms of the hardness of the target function classes.
\end{remark}

\begin{proof}[Proof of \prettyref{thm:genr}]
	Let $\calR$ denote the event in which \eqref{eq:delta-bound} holds. Then note that showing $\sqrt n(\hat \mu^\aipw-\mu)$ converges in distribution is equivalent to showing $\sqrt n(\hat \mu^\aipw-\mu)\indc{\calR}$ converges in distribution. This is because $\sqrt n(\hat \mu^\aipw-\mu)$ is bounded by $O(\sqrt n/\alpha)$ (in view of the boundedness of the estimators, the response functions and the outputs), and the difference of the above terms satisfies
	\begin{align*}
		\sqrt n\EE\qth{|\hat \mu^\aipw-\mu|\cdot\indc{\calR^c}}
		\leq O(\sqrt n/\alpha)\PP[\calR^c]=O(\sqrt n \xi/\alpha)\to 0.
	\end{align*}
	Hence, we will assume without a loss of generality that the event $\calR$ holds.

	For simplicity of notations, we note the following definitions
	\begin{align*}
			\pi_i^* = \pi^*(\bx_i),\quad \hat \pi_i = \hat \pi(\bx_i).
	\end{align*} 
	We first note the following decomposition
	\begin{align*}
		\hat \mu^\aipw-\mu
		=\frac 1n \sum_{i=1}^n S_i + R_0 + R_1+ R_2 + R_3,
	\end{align*}
	where 
	\begin{align}
		\begin{gathered}
			S_i={T_i\over \pi_i^*}\qth{y_i(1)-\mu_1^*(\bx_i)}
			- {1-T_i\over 1-\pi_i^*}\qth{y_i(0)-\mu_0^*(\bx_i)}
			+ \mu_1^*(\bx_i) - \mu_0^*(\bx_i) - \mu,
			\nonumber\\
			R_0 = \frac 1n\sum_{i=1}^n {T_i(y_i(1)-\mu_1^*(\bx_i))\over \hat \pi_i \pi_i^*}(\pi_i^*-\hat\pi_i),
			\quad
			R_1 = \frac 1n\sum_{i=1}^n {(1-T_i)(y_i(0)-\mu_0^*(\bx_i))\over (1-\hat \pi_i)(1-\pi_i^*)}(\pi_i^*-\hat\pi_i)
			\nonumber\\
			R_2 = \frac 1{n \hat \pi_i}\sum_{i=1}^n (\hat \pi_i-T_i)
			(\hat\mu_1(\bx_i)-\mu_1^*(\bx_i)),
			\quad
			R_3 = \frac 1{n(1-\hat \pi_i)}\sum_{i=1}^n (\hat \pi_i-T_i)
			(\hat\mu_0(\bx_i)-\mu_0^*(\bx_i)).
		\end{gathered}
	\end{align}
	We will show below that $\sqrt n R_i, i=0,1,2,3$, converges to zero in probability when the assumptions in the theorem statement are satisfied. Thus, the asymptotic normality of
	$\sqrt n(\hat \mu^\aipw-\mu )$ follows from the previous decomposition. We will use the following result.
	\begin{lemma} {\cite[Theorem 2.14.17']{vaart2023empirical}}\label{thm:empirical_process}
		Let $\calF$ be a class of measurable functions such that $\EE_P[h^2]\leq \delta^2,\|h\|_\infty\leq Q$ for every $h\in \calF$, and $\GG_n[h] = \sqrt n (\EE_n[h]-\EE_P[h])$ based on samples $\bx_1,\dots,\bx_n\simiid P$. Then there is a constant $c>0$ such that
		$$
		\EE\qth{\sup_{h\in \calF}|\GG_n[h]|}
		\leq c\tilde J_{[]}\pth{\delta,\calF,L_2(P)}
		\pth{1+{\tilde J_{[]}\pth{\delta,\calF,L_2(P)}\over \delta^2\sqrt n}Q}.
		$$
	\end{lemma}
	In view of the above theorem, we will proceed the proofs in the following way.
	\begin{enumerate}
		\item We first show $\sqrt n R_0$ converges to zero in probability and note that the convergence of $\sqrt n R_1$ can be shown similarly. 
		
		\item Next we show that $\sqrt n R_2$ converges to zero in probability if $\sqrt n\delta_1 \delta_2$ converges to zero and note that we can use a similar strategy to show that $\sqrt n R_3$  converges to zero in probability whenever $\sqrt n\delta_0 \delta_2$ converges to zero.
		
	\end{enumerate}
	We now jump into the details of the above two steps.
	\begin{enumerate}
		\item Note that $\sqrt{n} R_0 = \GG_n[h_{\hat \pi}]$, where 
		\begin{align*}
			h_{\pi}(\bx) = {T (y(1)- \mu_1^*(\bx))\over \pi(\bx)\pi^*(\bx)} \pth{\pi^*(\bx)-\pi(\bx)}
		\end{align*}
		and $\hat \pi \in \calF_2  \equiv \sth{ h_{\pi}: \EE_P\qth{(\pi-\pi^*)^2}\leq \delta_2^2, \pi\in \calF, \alpha<\pi<1-\alpha}$. Then for each fixed $ h_{\pi}\in \calF_2$ (given by a fixed $\pi\in \calF$)
		\begin{enumerate}[label=(\alph*)]
			\item $\EE_P[ h_{\pi}]=\EE_P\qth{\EE_P\qth{\left. h_{\pi}\right|\bx}}=0$, as given $\bx$, $y(1)-\mu_1^*(\bx)$ is uncorrelated with $T$ (unconfoundedness)
			\item $\EE_P\qth{ h_{\pi}^2}
			\leq {C\delta_2^2\over \alpha^4}$ for a constant $C$, as $y(1)-\mu_1^*(\bx)$ is uniformly bounded over all $\bx$ and we have assumed that $\alpha<\pi^*,\pi<1-\alpha$.
			\item $\| h_{\pi}\|_{\infty}\leq {C\over\alpha}$ for a constant $C>0$.
		\end{enumerate}
		Then using \prettyref{thm:empirical_process} with $\delta^2={C\delta_2^2\over \alpha^4},  Q={C\over \alpha}$ we get that $\EE\qth{\sup_{h_\pi \in \calF_2}\abs{\GG_n( h_\pi)}}$ converges to zero if the following are satisfied as $n\to \infty$
		\begin{align*}
			\tilde J_{[]}\pth{\delta,\calF,L_2(P)}\to 0,\quad
			{\tilde J_{[]}\pth{\delta,\calF,L_2(P)}\over \alpha \delta^2\sqrt n}
			\to 0
		\end{align*}
		which is guaranteed by our assumptions.  Consequently, 
		$$
		|\sqrt{n} R_0 |  =  \Bigl | \GG_n[h_{\hat \pi}] \Bigr |  \leq  \sup_{h_\pi \in \calF_2}\abs{\GG_n( h_\pi)} \stackrel{P}{\to} 0.
		$$

		\item 
		
		We next bound $\sqrt n R_2$. Denote
		\begin{align*}
			\tilde h_{\mu_1,\pi}(\bx) &= {(\pi^*(\bx)-T) \over \pi(\bx)}(\mu_1(\bx)- \mu_1^*(\bx)),
			\quad
			\breve h_{\mu_1,\pi}(\bx) &= {(\pi(\bx)-\pi^*(\bx)) \over \pi(\bx)}( \mu_1(\bx)- \mu_1^*(\bx))
		\end{align*}
		and note that 
		$\sqrt n R_2=\GG_n\qth{\tilde h_{\hat \mu_1,\hat \pi}}
		+\GG_n\qth{\breve h_{\hat \mu_1,\hat \pi}}$.
		Denote 
		\begin{align*}
			\calF_{11}&=\Biggl\{ \tilde h_{\mu_1,\pi}: \EE_P\qth{(\mu_1-\mu_1^*)^2}\leq \delta_1^2, \alpha<\pi<1-\alpha, \|\mu_1\|_{\infty}<\tilde B, \mu_1,\pi\in \calF\Biggr\},
			\nonumber\\
			\calF_{12}&=\Biggl\{\breve h_{\mu_1,\pi}: \EE_P\qth{(\mu_1-\mu_1^*)^2}\leq \delta_1^2,
			\EE_P\qth{(\pi-\pi^*)^2}\leq \delta_2^2,
			\nonumber\\
			&\quad ~~~~~
			\|\mu_1\|_{\infty}<\tilde B,\alpha<\pi<1-\alpha,
			\mu_1,\pi\in \calF\Biggr\}.
		\end{align*}
		Then, it suffices to separately show that 
		$$
		\lim _{n\to \infty}
		\sup_{\tilde h_{\mu_1,\pi}\in \calF_{11}}\abs{\GG_n\qth{\tilde h_{\mu_1,\pi}}} = 0,
		\quad
		\lim _{n\to \infty}
		\sup_{\breve h_{\mu_1,\pi}\in \calF_{12}}\abs{\GG_n\qth{\breve h_{\mu_1,\pi}}} = 0.
		$$ 
		
		\begin{enumerate}
			\item We first show $\lim _{n\to \infty}
			\sup_{\tilde h_{\mu_1}\in \calF_{11}}\abs{\GG_n\qth{\tilde h_{\mu_1,\pi}}} = 0$. For each fixed $\tilde h_{\mu_1,\pi}\in \calF_{11}$ (given by fixed $\mu_1,\pi\in \calF$)
			\begin{enumerate}
				\item $\EE_P[\tilde h_{\mu_1,\pi}]=\EE_P\qth{\EE_P\qth{\left. \tilde h_{\mu_1,\pi}\right|\bx}}=0$, as $\EE[T-\pi^*(\bx)|\bx]$ has expectation zero.
				\item $\EE_P\qth{ \tilde h_{\mu_1,\pi}^2}
				\leq {C\delta_1^2\over \alpha^2}$ for a constant $C$, as we assumed that $\alpha<\pi^*,\pi<1-\alpha$.
				\item $\|\tilde h_{\mu_1,\pi}\|_{\infty}\leq {C\over\alpha}$ for a constant $C>0$ as $\mu_1,\mu^*$ are bounded.
			\end{enumerate}
			Then we apply \prettyref{thm:empirical_process} with $\delta^2={C\delta_1^2\over \alpha^2},  Q={C\over \alpha}$ to get $\EE\qth{\sup_{\tilde h_{\mu_1,\pi} \in \calF_1}\abs{\EE_n(\tilde h_{\mu_1,\pi})}}$ converges to zero if the following are satisfied as $n\to \infty$
			\begin{align}\label{eq:m8}
				\tilde J_{[]}\pth{\delta,\calF\times \calF,L_2(P)}\to 0,\quad
				{\tilde J_{[]}\pth{\delta,\calF\times \calF,L_2(P)}\over \alpha \delta^2\sqrt n}
				\to 0
			\end{align}
			which is guaranteed by our assumptions as for any function classes $\calF_1,\calF_2$ 
			$$\tilde J_{[]}\pth{\delta,\calF_1\times \calF_2,L_2(P)}\leq \tilde J_{[]}\pth{\delta,\calF_1,L_2(P)}+\tilde J_{[]}\pth{\delta,\calF_2,L_2(P)}.$$
			
			\item 
			We establish $\lim _{n\to \infty}
			\sup_{\breve h_{\mu_1,\pi}\in \calF_{12}}\abs{\GG_n\qth{\breve h_{\mu_1,\pi}}} = 0$ using the following subparts
			\begin{align}\label{eq:m9}
				\lim _{n\to \infty}
				\sup_{\breve h_{\mu_1,\pi}\in \calF_{12}}\abs{\GG_n\qth{\breve h_{\mu_1,\pi}}} = 0,
				\quad
				\lim _{n\to \infty}
				\sqrt n\cdot \sup_{\breve h_{\mu_1,\pi}\in \calF_{12}}\abs{\EE_P\qth{\breve h_{\mu_1,\pi}}} = 0
			\end{align}
			For showing the first part, we note
			\begin{enumerate}
				\item $\EE_P\qth{ \breve h_{\mu_1,\pi}^2}
				\leq {C\delta_1^2\over \alpha^2}$ for a constant $C$, as we assumed that $\alpha<\pi^*,\pi<1-\alpha$.
				\item $\|\breve h_{\mu_1,\pi}\|_{\infty}\leq {C\over\alpha}$ for a constant $C>0$ as $\mu_1,\mu^*$ are bounded.
			\end{enumerate}
			Then we apply \prettyref{thm:empirical_process} with $\delta^2={C\delta_1^2\over \alpha^2},  Q={C\over \alpha}$ and \eqref{eq:m8} to get 
			$$\EE_P\qth{\sup_{\breve h_{\mu_1,\pi} \in \calF_{12}}\abs{\GG_n\qth{\breve h_{\mu_1,\pi}}}}\to 0.$$
			
			For the second expression in \eqref{eq:m9} using the Cauchy-Schwarz inequality we get
			$$
			\EE_P[\breve h_{\mu_1,\pi}]
			\leq \frac 1{\alpha} \sqrt{\EE_P\qth{(\pi-\pi^*)^2}\EE_P\qth{(\mu_1-\mu_1^*)^2}}
			$$
			for every fixed $\mu_1,\pi$ such that $\breve h_{\mu_1,\pi}\in \calF_{12}$. As all such $\mu_1,\pi$ satisfy 
			$$\EE_P\qth{(\mu_1-\mu_1^*)^2}\leq \delta_1^2,
			\EE_P\qth{(\pi-\pi^*)^2}\leq \delta_2^2,$$ 
			we continue the last display to get
			$
			\sqrt n \sup_{\breve h_{\mu_1,\pi} \in \calF_b}\abs{\EE \qth{\breve h_{\mu_1,\pi}}}\leq \sqrt n\delta_1\delta_2/\alpha,
			$
			which converges to zero in view of our assumptions.
		\end{enumerate}
	\end{enumerate}
\end{proof}

\begin{proof}[Proof of \prettyref{thm:main-ate}]
	We will verify the conditions in \prettyref{thm:genr}. Note that the conditions (i) and (ii) are satisfied in view of \prettyref{asmp:bounded}. Next we check condition (iii) in \prettyref{thm:genr}. In view of \cite[Theorem 4]{fan2024factor} we note that there exists a constant $c_1$ such that 
	$$
	\delta_{i,a} \leq \sup_{g\in \calH(r+|\calJ_i|,l,\calP)}\inf_{\hat g\in\calG(L,\barr+N,1,N,M,B)}\|g - \hat g\|_{\infty}^2\leq c_1 (n/\log n)^{-{2\gamma^*\over 2\gamma^*+1}},\quad i=0,1,2.
	$$
	In view of the definition of $\delta_{i,s},i=0,1,2$ in \prettyref{thm:main-fast-nn} we get that there exists a constant $c_2$ such that
	$$
	\delta_{i,s}\leq c_2(n/\log n)^{-{2\gamma^*\over 2\gamma^*+1}},\quad i=0,1,2.
	$$
	On the other hand, if $r\geq 1$, we use the assumption $p>(n/\log n)^{\frac 12 + c}$ for some constant $c>0$. This implies for constants $c_3>0$
	$$
	\delta_{i,f}\leq c_3(n/\log n)^{-(\frac 12 +c)},\quad i=0,1,2.
	$$
	Note that from the definition of $\delta_{i,f}$ in \prettyref{thm:main-fast-nn}, the above error becomes zero if $r=0$. As we do not require $p>(n/\log n)^{\frac 12 + c}$ for any other aspects of our proof, we can remove this requirement when $r=0$ and our proof for this specific case follows the remainder of the arguments.
	Hence, assuming $\gamma^*>\frac 12+c_4$ for some constant $c_4>0$, we get from \prettyref{thm:main-fast-nn} that there is an event $\calE$ with $\PP[\calE]\geq 1-n^{-2}$ and constant $c_5\in (0,\frac 12)$ such that
	\begin{align*}
		\EE_P[(\hat \mu_{j}^\fast - \mu_j^*)^2]
		&\leq (n/\log n)^{-\pth{\frac 12+c_5}},\quad j=0,1,\\
		\EE_P[(\hat \pi^\fast - \pi^*)^2]
		&\leq (n/\log n)^{-\pth{\frac 12+c_5}}.
	\end{align*}
	Hence, we get that on the event $\calE$, our estimators satisfy the requirement \eqref{eq:delta-bound} with $\delta_0^2=\delta_1^2=\delta_2^2=(n/\log n)^{-\pth{\frac 12+c_5}}$. Call this common value to be $\tilde \delta^2$. Then, the requirement (iii) in \prettyref{thm:genr} is satisfied. It remains to prove the condition (iv) in \prettyref{thm:genr}. In view of our choice of $\sth{\tau_k,\hat\bTheta_k}_{k=0}^2$ we first show that with a high probability
	\begin{align}\label{eq:psi-bound}
		\sum_{i,j}\psi_{\tau_k}(\hat\Theta_{k,i,j})\leq c_6 (n/\log n)^{\frac 12-c_5},\quad k=0,1,2.
	\end{align}
	We prove the case for $\tau_1$ as the other cases can be proven in a similar way. In view of \eqref{eq:m7}, choosing $t=2\log n$ we get that on an event $\calE_1$ with $\PP[\calE_1]\geq 1-\frac 1{n^2}$, for a constant $\bar c>0$
	\begin{align}\label{eq:m11}
		\sum_{i,j}\psi_{\tau_1}(\hat\Theta_{1,i,j})
		\leq \bar c \Biggl(|\calJ_1| + \frac 1{ \lambda_1}\Biggl\{
		&(n/\log n)^{-{2\gamma^*\over 2\gamma^*+1}}+({N^2L+N\bar r}){L\log(BNn)\over n}
		\nonumber\\
		&+{\log(np(N + \bar r)) + L\log(BN)\over n}\Biggr\}\Biggr).
	\end{align}
	As our choice of $\lambda_1$ guarantees $\lambda_1>{\log (np)\over n}$, we get the desired result. Here we used $N^2\asymp (n/\log n)^{1\over 2\gamma^*+1}<\sqrt {n\over \log n}$ from \prettyref{asmp:hyperparameters}. Now we are ready to verify the bracketing integral requirements in \prettyref{thm:genr}. We first note the following results for deep neural networks class that we use. A proof is provided later in this section.
	\begin{lemma}\label{lmm:bracketing-dnn}
		Consider the set $\calG_{m,s}$ defined as 
		\begin{align}\label{eq:G-ms}
			\calG_{m,s}=\sth{\mfunc=m^\fast(\cdot;\bW,g,\bTheta)\in \calG_m:\sum_{i,j}\psi_\tau(\Theta_{ij})\leq s,\|\bTheta\|_{\max}\leq B}.
		\end{align}
		and denote 
		$$\begin{gathered}
			A=LN^2+(L+1)N+N\barr+1,\\
			\tilde C = (M\vee K\|\bW\|_{\max})(L+1)B^L (N+1)^{L+1}
			+ K B^L N^L(N+\barr) p.
		\end{gathered}$$ 
		Then for all $\delta>0$
		\begin{align*}
			\tilde J_{[]}(\delta,\calG_{m,s},L_2(P))
			\leq
			8\sqrt{\log (B\tilde C)}\pth{\tau \log(1/\tau)\sqrt{A+Np} +\indc{\delta\geq \tau}\delta \log(1/\delta) \sqrt{(A+2s)}}.
		\end{align*}
	\end{lemma}

	We now apply \prettyref{lmm:bracketing-dnn} with $\alpha={1\over\log n},\delta={\tilde \delta\over \alpha},s=c_6 (n/\log n)^{\frac 12-c_5}$ and $\log p\leq (n/\log n)^{\frac 12 - \tilde c}$ for some constant $\tilde c\in (0,\frac 12)$ and obtain
	\begin{align*}
		(\log n)  {(\tilde J_{[]}(\delta,\calG_{m,s},L_2(P)))^2\over \delta^2 \sqrt n}
		\leq c_8(\log n)^4 \log (B\tilde C)\sth{{1\over \sqrt n}+{(A+2s)\log n\over \sqrt n}}
		\to 0 \text{ as } n\to \infty,
	\end{align*}
	and 
	\begin{align*}
		\tilde J_{[]}(\delta,\calG_{m,s},L_2(P))
		\leq c_9(\log n)^2\sqrt{\log (B\tilde C)}\sth{\delta+\sqrt{\delta^2(A+2s)\log n}}
		\to 0 \text{ as } n\to \infty. 
	\end{align*}
	where we used that $A\asymp N^2\asymp (n/\log n)^{1\over 2\gamma^*+1}<\sqrt {n\over \log n}$, given $\gamma^*>\frac 12$. 
	
\end{proof}

\begin{proof}[Proof of \prettyref{lmm:bracketing-dnn}]
	The proof strategy is as follows. Consider the parameter space 
	\begin{align*}
		\calU=\sth{(\bTheta,\{\bW_\ell,\bb_\ell\}_{\ell=1}^{L+1})\in [-B,B]^{Np+A}:
			m(\cdot;\bW,g,\bTheta)\in \calG_{m,s}}.
	\end{align*}
	We will first show a bound on $\log \calN(\epsilon_1,\calU,\|\cdot\|_{\infty})$ separately for the case $\epsilon_1>\tau$ and $\epsilon_1\leq \tau$. Then in view of \cite[Theorem 2.7.17]{vaart2023empirical} we can use the bound 
	\begin{align}\label{eq:bracket-entropy}
		\calN_{[]}(\epsilon_1,\calG_{m,s},L_2(\calP))
		\leq N\pth{{\epsilon_1\over 2\tilde C},\calU,\|\cdot\|_{\infty}},
	\end{align}
	where $\tilde C$ is as defined in the result statement, and satisfies \cite[Lemma 8]{fan2024factor}
	$$
	\sup_{\bx\in [-K,K]^p}|m(\bx)-\breve m(\bx)|\leq \tilde C \|\btheta(m) - \btheta(\breve m)\|_{\infty},
	\quad
	m,\breve m\in \calG_{m,s},
	\btheta(m)=\{\bTheta,(\bW_\ell,\bb_\ell)_{\ell=1}^{L+1}\}.
	$$
	\begin{enumerate}
		\item Consider the case {$\epsilon_1\geq\tau$}: We will use the cover
		\begin{align*}
			\calU(\epsilon_1)=\cup_{S\subset [N]\times[p]:|S|=s}\calU(\epsilon_1, S)
		\end{align*}
		where
		\begin{align*}
			\calU(\epsilon_1, S) = \Bigg\{&[\bW_\ell]_{i,j}, [\bb_{\ell}]_{j} \in \left\{ -B+\epsilon_1,\cdots, -B+\epsilon_1 \cdot \left\lceil \frac{2B}{\epsilon_1} \right\rceil \right\},  \\
			&~\bTheta_{S} \in \left\{-B+\epsilon_1,\cdots, -B+\epsilon_1 \cdot \left\lceil \frac{2B}{\epsilon_1} \right\rceil \right\}^{s}, \bTheta_{S^c} = 0 \Bigg\}
		\end{align*}
		Then we show that $\calU(\epsilon_1)$ is a valid $\epsilon_1$-cover of $\calU$ in the $\|\cdot\|_{\infty}$ norm. Note that 
		$$
		\sum_{i,j} \indc{|\Theta_{i,j}|>\tau}\leq
		\sum_{i,j} \psi_\tau(\Theta_{ij})\leq s.
		$$
		The above implies, given any $T\in \calU$, there is a set $S\subset [n]\times [p]$ with $|S|\leq s$ such that $\bTheta_{S^c}$ has all entries with absolute value bounded by $\tau$. As $\epsilon_1\geq \tau$, we can find a $\tilde T\in \calU(\epsilon_1,S)$ such that $\|\tilde T-T\|_{\infty}\leq \epsilon_1$. Hence, $\calU(\epsilon_1)$ gives us an $\epsilon_1$-cover of $\calU$. Note that the total number of parameters in $\{\bW_\ell,\bb_\ell\}_{\ell=1}^{L+1}$ is $N(N+\barr)+N + \sum_{\ell=2}^L N(N+1) + N+1$, which is defined as $A$ in the lemma statement. Then, it is straightforward to check that $\calU(\epsilon_1)$ has at most $\ceil{2B\over \epsilon_1}^{s+A}$ entries. Next, note that the number of choices for $S\subset[N]\times [p]$ such that $|S|=s$ is $\binom{Np}{s}\leq (Np)^s$. Hence
		\begin{align*}
			\calN(\epsilon_1,\calU,\|\cdot\|_{\infty})
			\leq \ceil{2B\over \epsilon_1}^{s+A}
			(Np)^s
			\leq \ceil{2B\over \epsilon_1}^{s+A}
			(\tilde C)^s,\quad
			\epsilon_1\geq \tau.
		\end{align*}
		\item Next, we consider {$\epsilon_1<\tau$}:
		\begin{align*}
			\calU(\epsilon_1) = \Bigg\{&[\bW_\ell]_{i,j}, [\bb_{\ell}]_{j},[\bTheta]_{i,j} \in \left\{ -B+\epsilon_1,\cdots, -B+\epsilon_1 \cdot \left\lceil \frac{2B}{\epsilon_1} \right\rceil \right\} \Bigg\}
		\end{align*}
		It is straightforward to show that $\calU(\epsilon_1)$ gives an $\epsilon_1$-cover of $\calU$. As the total number of parameters in $(\bTheta,\{\bW_\ell,\bb_\ell\}_{\ell=1}^{L+1})$ is $Np+N(N+\barr)+N + \sum_{\ell=2}^L N(N+1) + N+1=Np+A$, $\calU(\epsilon_1)$ has at most $\ceil{2B\over \epsilon_1}^{Np+A}$ entries. Hence
		\begin{align*}
			\calN(\epsilon_1,\calU,\|\cdot\|_{\infty})
			\leq \ceil{2B\over \epsilon_1}^{Np+A},\quad
			\epsilon_1 < \tau.
		\end{align*}
	\end{enumerate}
	Combining the above, in view of \eqref{eq:bracket-entropy} we get that
	\begin{align*}
		\log\pth{\calN_{[]}(\epsilon_1,\calG_{m,s},L_2(\calP))}
		\leq
		\Biggl\{\indc{\epsilon_1<\tau}\pth{A+Np}
		+\indc{\epsilon_1\geq \tau}\pth{A+2s}\Biggr\}\log\pth{1+4B\tilde C\over \epsilon_1}.
	\end{align*}
	As $B\tilde C$ is large, for all $\epsilon_1<1$, we can use the inequality $\log\pth{1+4B\tilde C\over \epsilon_1}\leq 4\log(B\tilde C)\log(1/\epsilon_1)$.
	Noting the definition of the bracketing integral in \prettyref{def:bracketing} we get
	\begin{align}\label{eq:m6}
		&\tilde J_{[]}(\delta,\calG_{m,s},L_2(P))
		\nonumber\\
		&\leq
		4\sqrt{\log (B\tilde C)}\pth{\indc{\delta<\tau}\sqrt{A+Np} \int_{0}^\tau \sqrt{\log(1/\varepsilon)} d\varepsilon +\indc{\delta\geq \tau} \sqrt{(A+2s)} \int_{0}^\delta \sqrt{\log(1/\varepsilon)} d\varepsilon}.
	\end{align}
	Hence, it suffices to bound $\int_0^\delta \sqrt{\log(1/\epsilon)}\ d\epsilon$. Using a change of variable $\log(1/\epsilon)=z^2$ the integral can be transformed into $2\int_{\nu}^\infty z^2e^{-z^2} dz$ with $\nu=\sqrt{\log(1/\delta)}$. To study the last integral we will use $\int_{\nu}^\infty e^{-(az)^2} dz$. Note that for any $a\neq 0$ we have
	\begin{align*}
		\int_{\nu}^\infty e^{-(az)^2} dz 
		&= \frac 1a \int_{a\nu\sqrt 2}^\infty 2^{-\frac {y^2}2}dy
		={\sqrt{2\pi}\over a}\pth{1-\Phi(a\nu\sqrt 2)}.
	\end{align*}
	Differentiating the above display with respect to $a$ we get ($\phi$ is the standard Gaussian density)
	\begin{align*}
		2a\int_{\nu}^\infty z^2 e^{-(az)^2} dz 
		={\sqrt{2\pi}\over a^2}\pth{1-\Phi(a\nu\sqrt 2)}
		+{\sqrt{2\pi}\over a}\phi(a\nu\sqrt 2)\nu\sqrt 2.
	\end{align*}
	Plugging in $a=1$ and using Mill's ratio bound $1-\Phi(x)\leq {\phi(x)\over x}$ for $x>0$ we get
	\begin{align*}
		2\int_{\nu}^\infty z^2 e^{-z^2} dz 
		\leq {\sqrt{2\pi}\phi(\nu\sqrt 2)\over \nu\sqrt 2}
		+{\sqrt{2\pi}}\phi(\nu\sqrt 2)\nu\sqrt 2
		\leq {e^{-\nu^2}\over \nu\sqrt 2}+{\nu\sqrt 2 e^{-\nu^2}}.
	\end{align*}
	Finally substituting $\nu=\sqrt{\log(1/\delta)}$ we get
	\begin{align*}
		\int_0^\delta \sqrt{\log(1/\epsilon)}\ d\epsilon
		\leq \delta\pth{\sqrt{2\log\frac 1\delta}+\frac 1{\sqrt{2\log\frac 1\delta}}}
		\leq 2\delta\log(1/\delta).
	\end{align*}
	Then, in view of \eqref{eq:m6}, the result follows. This completes the proof.
\end{proof}

\section{Proof of \prettyref{thm:main-fast-nn}}
Proving the result related to $\pi^*$ is similar to the proof of \cite[Theorem 2]{fan2024factor} as the proofs depend on the entire sample space, so it is omitted here. We prove the result related to estimating $\mu_0^*,\mu_1^*$ and the proof differs from the standard functional guarantees in \cite[Theorem 2]{fan2024factor} as the above work analyses function estimation with fixed data and our estimators for $\mu_0^*,\mu_1^*$ are constructed using the random subsamples given by the control group $\{(y_i,\bx_i,T_i):T_i=0\}$ and the treatment group $\{(y_i,\bx_i,T_i):T_i=1\}$ respectively. Our proof will rely on the following auxiliary result. Let $P_j$ denote the conditional law of $\bx$ given $T=j,j=0,1$ and define
$$
\EE_{P_j}[h]=\int h(\bx) dP_j(\bx),
\quad
\EE_{n,j}[h] = \frac 1{n_j}\sum_{i:T_i=j} h(\bx_i),
\quad
n_j=\sum_{i=1}^n \indc{T_i=j},\quad
j=0,1.
$$

\begin{lemma}\label{lmm:conditional-bound}
	Suppose that the conditions in \prettyref{thm:main-fast-nn} hold. Then, with probability at
	least $1 - O(e^{-t}+e^{-n\alpha_*^2/2})$, the following holds, for $n$ large enough and $j=0,1$
	\begin{align*}
		&\EE_{P_j}\qth{ \pth{\hat \mu_{j}^\fast- \mu_j^*}^2}+ \EE_{n,j}\qth{(\hat \mu_{j}^\fast - \mu_j^*)^2}
		\leq {c\over \alpha_*}\sth{\delta_\opt + \delta_{j,a} + \delta_{j,s} + \delta_{j,f} +\frac tn},
	\end{align*}
	where $c$ is a constant.
\end{lemma}

In view of the above result, the proof of \prettyref{thm:main-fast-nn} relies on bounding $\EE_P$ using the conditional expectation $\EE_{P_j}$, as \prettyref{asmp:reg} implies that $dP(\bx)={\PP[T=j]\over \PP[T=j|\bx]} dP_j(\bx)
	\leq {dP_j(\bx)\over \alpha_*}$, for each $j=0,1$.
Hence, it remains to prove \prettyref{lmm:conditional-bound}.

\noindent \textbf{Notations:} We use the following notation for the proofs in this section. Given any matrix $\bB$ with $n$ rows and a subset $\calJ$ of the index set $\{1,\dots,n\}$, let $[\bB]_{\calJ,:}$ denote the submatrix consisting of the rows corresponding to the $\calJ$ index set. In view of \eqref{eq:factor}, define
\begin{align}\label{eq:diversified-decomposition}
	\tilde \bff= \frac 1p \bW^\top \bx, \quad
	\bH=\frac 1p \bW^\top \bB.
\end{align}
Given a matrix $\bH\in \reals^{\bar r\times r}, \bar r\geq r$ with full column rank define $\bH^+$ to be it's left inverse
$
	\bH^+=(\bH^\top\bH)^{-1}\bH^\top.
$
Assume that $\mfunc^*(\bx)=\mfunc(\bff,\bu_{\calJ_1})$, i.e., the coordinates corresponding to $\calJ_1$ are active in the output function $\mfunc^*$. Also define
\begin{align}
	\label{eq:v-rho}
	v_n=({N^2L+N\bar r}){L\log(BNn)\over n},
	\quad
	\varrho_n={\log(np(N + \bar r)) + L\log(BN)\over n}.
\end{align}

\begin{proof}[Proof of \prettyref{lmm:conditional-bound}]
	
	We only prove results related to $\mu_1^*$ as the result for $\mu_0^*$ is similar. We first outline the key steps of the proof of \prettyref{lmm:conditional-bound}.
	
	\begin{itemize}
		\item \textbf{Step 1:} Show that $\tilde \mfunc^*(\bx) = \mfunc^*(\bH^+\tilde \bff,\bx_{\calJ_1}-[\bB]_{\calJ_1,:}\bH^+\tilde\bff)$ is close to $\mfunc^*(\bff,\bu_{\calJ_1})$
		\begin{align}
			\EE_{P}\qth{\pth{\tilde \mfunc^*-\mfunc^*}^2}
			\lesssim {|{\calJ_1}| r \cdot \bar r\over (\nu_{\min}(\bH))^2p}
			=\delta_{1,f},
		\end{align}
		where $\EE_P$ is the expectation with respect to the unconditional distribution of $\bx$. Then noting that $\pi^*(\bx)=\PP[T=1|\bx]\in (\alpha_*,1-\alpha_*)$ for all $\bx$, we get
		\begin{align}
			\label{eq:step1}
			\EE_{P_1}\qth{\pth{\tilde \mfunc^*-\mfunc^*}^2}
			\leq {\EE_{P}\qth{\pth{\tilde \mfunc^*-\mfunc^*}^2}\over \alpha_*}
			=\delta_{1,f}/\alpha_*.
		\end{align}
		\item \textbf{Step 2:} Define the function class
		\begin{align}\label{eq:Gm}
			\calG_m = \sth{m^{\fast}(\bx; \bW,g,\bTheta): g\in \calG(L,\bar r +N,1,N,M,B),\bTheta\in \reals^{p\times N},\|\bTheta\|_{\max}\leq B}
		\end{align}
		Then there exists $\tilde \mfunc\in \calG_m$ (i.e., with corresponding $\tilde \bTheta_1$) such that $\|\tilde \bTheta_1\|_0\leq |\calJ_1|$ and 
		\begin{align}
			\EE_{P}\qth{\pth{\tilde \mfunc-\tilde \mfunc^*}^2}
			\lesssim \delta_{1,f}+\delta_{1,a}.
		\end{align}
		Similar to \eqref{eq:step2} we get $\EE_{P_1}\qth{\pth{\tilde \mfunc-\tilde \mfunc^*}^2}\lesssim (\delta_{1,f}+\delta_{1,a})/\alpha_*$,
		which implies
		\begin{align}\label{eq:step2}
			\EE_{P_1}\qth{(\tilde\mfunc-\mfunc^*)^2}
			\leq 2\pth{\EE_{P_1}\qth{(\tilde \mfunc-\tilde \mfunc^*)^2}+
				\EE_{P_1}\qth{(\tilde \mfunc^*- \mfunc^*)^2}}
			\lesssim (\delta_{1,f}+\delta_{1,a})/\alpha_*.
		\end{align}
		
		\item \textbf{Step 3:} Derive the basic inequality 
		\begin{align}\label{eq:step3}
			&~\EE_{n,1}[(\hat \mfunc^\fast-\tilde \mfunc)^2]
			+2\lambda_1 \sum_{i,j}\psi_{\tau_1}(\hat\Theta_{1,i,j})
			\nonumber\\
			&\leq 
			4\EE_{n,1}[(\tilde \mfunc -\mfunc^*)^2]
			+\frac 4{n_1} \sum_{i\in [n]:T_i=1}\varepsilon_i(1) (\hat \mfunc^\fast(\bx_i)-\tilde \mfunc(\bx_i))
			+2\lambda_1|\calJ_1|+2\delta_{\opt}.
		\end{align}

		\item \textbf{Step 4:} Show that the following event occurs with a probability at least $1-e^{-t}$
		\begin{align}\label{eq:step4}
			\calB_{t,1/2}
			=\Bigg\{&\forall \mfunc=m^\fast(\cdot; \bW,g,\bTheta)\in \calG_m,\
			\frac 4{n_1} \sum_{i\in [n]:T_i=1} \varepsilon_i(1) (\mfunc(\bx_i)-\tilde \mfunc(\bx_i))
			\nonumber\\
			&~ 
			-\lambda_1\sum_{i,j}\psi_\tau(\Theta_{1,i,j})\leq \frac 12 \EE_{n,1} [(\mfunc-\tilde \mfunc)^2]+2\pth{{v_{n_1}}+\varrho_{n_1}+\frac t{n_1}}\Bigg\}. 
		\end{align}
		
		\item \textbf{Step 5:} Show that the following event occurs with a probability $1-e^{-t}-e^{-n\alpha_*^2/2}$
		\begin{align}\label{eq:step5}
			\calC_t
			=\Bigg\{\forall \mfunc&=m^\fast(\cdot; \bW,g,\bTheta)\in \calG_m,\ \frac 12\EE_{P_1}\qth{\pth{\mfunc-\tilde \mfunc}^2}
			\nonumber\\
			&\leq \frac 1{2}\EE_{n,1}[(\mfunc-\tilde \mfunc)^2]
			+{2\lambda_1}\sum_{i,j}\psi_{\tau_1}(\Theta_{i,j})+C_5\pth{v_{n_1}+\varrho_{n_1}+\frac t{n_1}}\Bigg\}. 
		\end{align}
		
		\item \textbf{Step 6:} We bound the separation between $\mfunc^*$ and $\tilde \mfunc$ from Step 2. For every $0<t\leq n$, there is an event $\calA_t$ with $\PP[\calA_t]\geq 1-e^{-t}-e^{-n\alpha_*^2/2}$ on which
		\begin{align}\label{eq:step6}
			&\EE_{n,1}[(\tilde \mfunc-\mfunc^*)^2]
			\lesssim \frac 1{\alpha_*}\pth{{\delta_{1,f}+\delta_{1,a}}+\frac tn}.
		\end{align} 
		
		\item \textbf{Step 7:} We bound $\EE_{n,1}[(\hat \mfunc^\fast-\mfunc^*)^2]$. Using \eqref{eq:step3}, \eqref{eq:step4} and \eqref{eq:step6} we get on the event $\calB_{t,1/2}$
		\begin{align}
			\label{eq:step7-1}
			&~\EE_{n,1}[(\hat \mfunc^\fast-\tilde \mfunc)^2]
			+2\lambda_1 \sum_{i,j}\psi_{\tau_1}(\hat\Theta_{1,i,j})
			\nonumber\\
			&\leq 
			4\lambda_1|\calJ_1|+4\delta_\opt+{C_4}\pth{{v_{n_1}}+\varrho_{n_1}+\frac t{n_1}}
			+{\tilde C_4\over \alpha_*}\pth{\delta_{1,f}+\delta_{1,a}+\frac tn}.
		\end{align}
		Combining the last display and \eqref{eq:step6} with the following facts
		\begin{itemize}
			\item $(\hat \mfunc^\fast-\mfunc^*)^2
			\leq 2[(\tilde \mfunc-\mfunc^*)^2+(\hat \mfunc^\fast-\tilde \mfunc)^2]$, 
			\item for $n_1\in (n\alpha_*/2,n)$
			\begin{align}\label{eq:m10}
				v_{n_1}\lesssim {v_n\over \alpha_*},
				\quad
				\varrho_{n_1}\lesssim {\varrho_n\over \alpha_*} 
			\end{align} 
		\end{itemize} 
		we get on $\calE_1=\calA_t\cap \calB_{t,1/2}\cap \calD$ ($\calD$ is as in \prettyref{lmm:count}) with $\PP[\calE_1]\geq 1-e^{-t}-e^{-n\alpha_*^2/2}$
		\begin{align}\label{eq:step7}
			\EE_{n,1}[(\hat \mfunc^\fast-\mfunc^*)^2]
			\lesssim \lambda_1|\calJ|+\delta_\opt
			+ \frac 1{\alpha_*}\pth{{v_{n}+\varrho_{n}+\delta_{1,f}+\delta_{1,a}}+\frac t{n}}.
		\end{align} 
		
		\item \textbf{Step 8:} We bound $\EE_{P_1}[{(\hat \mfunc^\fast-\mfunc^*)^2}$]. On the event $\calC_t\cap \calD$ we use \eqref{eq:step7-1} to get
		\begin{align}\label{eq:m7}
			\EE_{P_1}\qth{\pth{\hat \mfunc^\fast-\tilde \mfunc}^2}
			&\leq \EE_{n,1}[(\hat \mfunc^\fast-\tilde \mfunc)^2]
			+{4\lambda_1}\sum_{i,j}\psi_{\tau_1}(\hat \Theta_{1,i,j})+{C_5}\pth{v_{n_1}+\varrho_{n_1}+\frac t{n_1}}
			\nonumber\\
			&\lesssim
			\frac 1{\alpha_*}\pth{\lambda_1|\calJ_1|+\delta_{1,f}+\delta_\opt+ {v_{n}}+\varrho_{n}+\frac tn}.
		\end{align}
		where the last inequality followed using \eqref{eq:m10}. We continue the last display using $\|\ba+\bb\|_2^2\leq 2(\|\ba\|_2^2+\|\bb\|_2^2)$ and \eqref{eq:step2} to get on the event $\calE_2=\calB_{t,1/2}\cap \calC_t\cap \calD$ with $\PP[\calE_1]\geq 1-e^{-t}-e^{-n\alpha_*^2/2}$
		\begin{align}\label{eq:step8}
			\EE_{P_1}\qth{\pth{\hat \mfunc^\fast-\mfunc^* }^2}
			&\leq 2\pth{\EE_{P_1}\qth{\pth{\hat \mfunc^\fast-\tilde \mfunc }^2}+\EE_{P_1}\qth{\pth{\tilde \mfunc-\mfunc^*}^2}}		\nonumber\\
			&\lesssim
			\frac 1{\alpha_*}\pth{\lambda_1|\calJ_1|+\delta_\opt+ {v_{n}}+\varrho_{n} +\delta_{1,f}+\delta_{\opt}+\frac tn}.
		\end{align}
	\end{itemize}

	To complete the proof of \prettyref{lmm:conditional-bound} we only prove steps 1--6, as the other steps follow from them via simple algebra. The proof of Steps 1-2 follows from the proof of \cite[Theorem 2]{fan2024factor} and uses properties of the functions $\tilde \mfunc,\tilde \mfunc^*,\mfunc^*$. Next, we outline the proof of Step 3. Note that from the definition of $\hat g_1,\hat \bTheta_1$ in \prettyref{eq:mu1-loss} it follows
	\begin{align*}
		&~\frac 1{n_1}\sum_{i\in [n], T_i=1} \sth{y_i - \hat \mfunc^{\fast}(\bx_i)}^2 
		+\lambda_1\sum_{i,j}\psi_{\tau_1}(\hat \Theta_{1,i,j})
		\nonumber\\
		&\leq
		\frac 1{n_1}\sum_{i\in [n], T_i=1} \sth{y_i - \tilde \mfunc(\bx_i)}^2 
		+\lambda_1\sum_{i,j}\psi_{\tau_1}(\tilde \Theta_{1,i,j}) +\delta_\opt
	\end{align*}
	Substituting $y_i=\mu_1^*(\bx_i)+\varepsilon_i(1),T_i=1$ in the above expression, we get 
	\begin{align*}
		&~\EE_{n,1}[(\mfunc^*-\hat\mfunc^\fast)^2]+\lambda_1\sum_{i,j}\psi_{\tau_1}(\hat \Theta_{1,i,j})
		\nonumber\\
		&\leq
		\EE_{n,1}[(\mfunc^*-\tilde\mfunc)^2]
		+\frac 2{n_1}\sum_{i\in [n]:T_i=1}\varepsilon_i(1)(\mfunc^\fast(\bx_i)-\tilde\mfunc(\bx_i))
		+\lambda_1\sum_{i,j}\psi_{\tau_1}(\tilde \Theta_{1,i,j}) +\delta_\opt	\end{align*} 
	In view of the construction of $\tilde \bTheta_1$ in Step 2 we have
	$$
	\sum_{i,j}\psi_{\tau_1}(\tilde \Theta_{1,i,j})\leq 
	\|\tilde \bTheta_1\|_0\leq |\calJ_1|.
	$$ 
	Using the last display and $\frac 12(\hat \mfunc^\fast-\tilde \mfunc)^2\leq (\hat \mfunc^\fast-\mfunc^*)^2+(\mfunc^*-\tilde \mfunc)^2$ we can rearrange the expressions to derive \prettyref{eq:step3}. The proof of Step 4 follows from \cite[Lemma 10]{fan2024factor} along with a union bound argument. In view of the proof of \cite[Lemma 10]{fan2024factor} we note that by defining
	\begin{align*}
		&~\calB_{t,1/2}(\{\bx_i:T_i=1, i\in [n]\})
		\nonumber\\
		&=\Bigg\{\forall \mfunc=m^\fast(\cdot; \bW,g,\bTheta)\in \calG_m,\
		\frac 4{n_1} \sum_{i\in [n]:T_i=1} \varepsilon_i(1)(\mfunc(\bx_i)-\tilde \mfunc(\bx_i))
		\nonumber\\
		&-\lambda_1\sum_{i,j}\psi_\tau(\Theta_{1,i,j})
		~\leq \frac 1{2n_1}\sum_{i\in [n]:T_i=1} (\mfunc(\bx_i)-\tilde \mfunc(\bx_i))^2+2\pth{{v_{n_1}}+\varrho_{n_1}+\frac t{n_1}}\Bigg\},
	\end{align*}
	we get $\PP[\calB_{t,1/2}(\{\bx_i:T_i=1\})]\geq 1-e^{-t}$ holds for every fixed realization of $\{\bx_i:T_i=1\}$. Hence, we can then apply the Law of Total Probability to conclude the statement. 
	
	To prove Step 5, define $\calI=\{i\in [n]: T_i=1\}$ and let $|\calI|$ denote the size of $\calI$. Note that conditioned on a fixed realization of $\calI$, the points $\{\bx_i\}_{i\in \calI}$ are independently distributed with the distribution $P_1$. Next, we restrict ourselves to the event $\calD=\{\sum_{i\in [n]}T_i\geq {n\alpha_*/2}\}$, which occurs with a probability at least $1-e^{-n\alpha_*^2/2}$ in view of \prettyref{lmm:count}. In view of the above, we can first show that for each fixed realization from the event $\calD$,  the following event holds with a probability $1-e^{-t}$, for any fixed $\tilde \mfunc\in \calG_m$
	\begin{align*}
		\calC_{t,\calI}
		=\Bigg\{\forall \mfunc&=m^\fast(\cdot; \bW,g,\bTheta)\in \calG_m,\ \frac 12\EE_{P_1}\qth{\pth{\mfunc-\tilde \mfunc}^2}
		\nonumber\\
		&\leq \frac 1{2|\calI|}\sum_{i\in \calI}(\mfunc(\bx_i)-\tilde \mfunc(\bx_i))^2
		+{2\lambda_1}\sum_{i,j}\psi_{\tau_1}(\Theta_{i,j})+C_5\pth{v_{|\calI|}+\varrho_{|\calI|}+\frac t{|\calI|}}\Bigg\}. 
	\end{align*}
	In addition, by further conditioning on the event in \prettyref{lmm:count} we get that $|\calI|$ is of constant order compared to $n$. Then we can follow the proof of \cite[Lemma 9]{fan2024factor} to show that
	\begin{align*}
		\PP\qth{\calC_{t,\calI}\cap \calD}\geq 1-e^{-t}-e^{-n\alpha_*^2/2}.
	\end{align*}
	Hence, using the law of total probability we get
	$$
	\PP[\calC_{t,\calI}]\geq \PP[\calC_{t,\calI}|\calD]\cdot \PP[\calD]\geq (1-e^{-t})(1-e^{-n\alpha_*^2/2})
	\geq 1-e^{-t}-e^{-n\alpha_*^2/2}.
	$$
	
	We present the proof of Step 6 below. We will first apply \cite[Lemma 9, Lemma 10]{fan2024factor} based on every fixed realization from the event $\calD$ as in \prettyref{lmm:count}. Note that the random variables $\{\bx_i:i\in \calI\}$ are independently and identically distributed. This implies that for $\mfunc^*$ and $\tilde \mfunc\in \calG_m$ as in Step 2, the following collection of random variables
	$$
	z_i=(\tilde \mfunc(\bx_i)-\mfunc^*(\bx_i))^2,\quad i\in [n], T_i=1
	$$ 
	are independent and satisfies (as $\tilde \mfunc\in [-M,M],\mfunc^*\in [-M^*,M^*]$ using \prettyref{asmp:reg})
	\begin{align*}
		z_i\leq (M+M^*)^2,\ \EE_{P_1}[z_i^2]\leq (M+M^*)^2\EE_{P_1}[(\tilde \mfunc-\mfunc^*)^2]
		\lesssim {\delta_{1,f}+\delta_{1,a}\over \alpha_*},
	\end{align*}
	Hence, conditioned on $n_1$ we can apply the Bernstein Inequality \citep{boucheron2003concentration} to conclude that with a probability $1-e^{-t}$
	\begin{align}\label{eq:m4}
		\EE_{n,1}[(\tilde \mfunc-\mfunc)^2]
		&\leq 
		\EE_{P_1}[(\tilde \mfunc-\mfunc^*)^2]+
		C\pth{{\delta_{1,f}+\delta_{1,a}\over \alpha_*}\sqrt{t\over n_1}+\frac t{n_1}},
	\end{align}
	for a constant $C>0$. Next we show that $n_1\geq {n\alpha_*\over 2}$ with a probability at least $1-e^{-{n\alpha_*^2\over 2}}$. 
	In view of \eqref{eq:m4} we use the last display and use a union bound to conclude that there exists a constant $C>0$ such that for all $0<t\leq n$
	\begin{align*}
		\PP\qth{\EE_{n,1}[(\tilde \mfunc-\mfunc^*)^2]\leq 
			{C\over \alpha_*}\pth{{\delta_{1,f}+\delta_{1,a}}+\frac t{n}}}
		\geq 1-e^{-t}-e^{-n\alpha_*^2/2}.
	\end{align*}
\end{proof}
\begin{lemma}\label{lmm:count}
	Define the event $\calD=\{\sum_{i\in [n]} T_i\geq n\alpha_*/2\}$. Then $\PP[\calD]\geq 1-e^{-n\alpha_*^2/2}$.
\end{lemma}
\begin{proof}
	Note that as $\inf_{\bx}\pi^*(\bx)\geq \alpha_*$ we get $\sum_{i\in [n]}{T_i}$ are stochastically larger than $Z\sim\Binom(n,\alpha_*)$. Hence, we get for $c>1$ to be chosen later
	\begin{align}\label{eq:m0}
		\PP[n_1\leq {n\alpha_*\over c}]
		\leq \PP[Z\leq {n\alpha_*\over c}]
		\leq \PP[n-Z\geq n\pth{1-{\alpha_*\over c}}]
	\end{align}
	Here $n-Z\sim \Binom(1-\alpha_*)$. We will use Chernoff's inequality for Binomial random variables
	\begin{lemma}{\cite[Section 2.2]{boucheron2003concentration}}
		\label{lmm:chernoff}
		For a random variable $Z\sim \Binom(m,q)$, we have
		\begin{align*}
			\PP\qth{Z\geq ma}
			\leq \exp\pth{-mh_q(a)};\quad q<a<1,\
			h_q(a)=a\log{a\over q}+(1-a)\log{1-a\over 1-q}.
		\end{align*}
	\end{lemma}
	Using $q=1-\alpha_*,a=1-\alpha_*/c$ in the definition of $h_q(a)$ in the above result and using Pinsker's inequality $h_q(a)\geq 2(a-q)^2$ we get $h_q(a)\geq 2\alpha_*^2(c-1)^2/c^2$. Hence we continue \prettyref{eq:m0} using \prettyref{lmm:chernoff} to get 
	\begin{align*}
		\PP\qth{n_1\leq {n\alpha_*\over c}}
		\leq \exp\pth{-2n\alpha_*^2(c-1)^2/c^2}.
	\end{align*}
	Plugging in $c=2$ in the above inequality, we get the desired result.
\end{proof}

\begin{proof}[Proof of \prettyref{thm:main-fast-nn} for $r=0$]
	The proof here mainly deviates from the proof of the case $r\geq 1$ in establishing \eqref{eq:step1} and \eqref{eq:step2}. We modify the steps as follows. Note that we have $\delta_{1,f}=0$ and $r=0$. We also get from $r=0$ that $\mu_1^*(\bx)=\mu_1^*(\bx_{\calJ_1})$. Then, from the definition of $\delta_{1,a}$ as in \prettyref{thm:main-fast-nn} note that there exists $\tilde \mfunc\in \calG_m$ (i.e., with corresponding $\tilde \bTheta_1$, and $\calG_m$ is as defined in \eqref{eq:Gm}) such that $\|\tilde \bTheta_1\|_0\leq |\calJ_1|$ and $
		\EE_{P}\qth{\pth{\tilde \mfunc-\mfunc^*}^2}
		\lesssim \delta_{1,a}.$
	Similar to \eqref{eq:step2}, we get $\EE_{P_1}\qth{\pth{\tilde \mfunc-\tilde \mfunc^*}^2}\lesssim \delta_{1,a}/\alpha_*$. Then the other parts of the proof, from Step 3 onwards, can be carried out as before. This completes our argument.
\end{proof}

\section{Empirical Implementation Details}
\label{app:sim-imp}

\subsection{Parameter choices for candidate methods}
The scripts to submit each simulation as a job on the cluster are named identically to the file for the corresponding Python codes with an extension of `\textit{.sh}'. All the directories in the Python code are saved relatively, so people can execute the code under any directory without changing the paths inside the Python scripts. GPUs are recommended to simulate Double Deep Learning, Vanilla Neural Networks with $L_2$ regularization. 
\begin{itemize}
\item Factor Informed Double Deep Learning Estimator ({FIDDLE}): We implement a factor-augmented sparse throughput deep (FAST) ReLU neural network to estimate the average treatment effect (ATE). Set the number of epochs in training to be $100$, the batch size to be $64$, the learning rate $lr = 0.001$, the depth of the neural network $L = 4$, and the width of the neural network $N = 400$ and the column number of the diversified projection matrix $\overline{r} = 10$. The hyperparameters for the penalty in the FAST architecture are set to $\tau = 0.005$ and $\lambda = 1.3 \, \mathrm{log}(p) / n$. We randomly sample $m = 50$ unlabeled observations of covariates to pre-train the diversified projection matrix $\boldsymbol{W}$ and leave the rest of the dataset to estimate the propensity and outcomes models. The column number of the diversified projection matrix $\overline{r} = 3$ only for the experiments on the MBSAQIP dataset in Section \ref{ss:application-surgery}, due to a smaller number of covariates. 
\item Vanilla Neural Networks (Vanilla-NN): We adopt a fully connected ReLU neural network with the same number of epochs in training to be $100$, the batch size $64$, and the learning rate $lr = 0.001$, the depth of the neural network $L = 4$ and the width of the neural network $N = 400$. We penalize the loss function by an $L_2$ norm term with weight $\lambda = 1$. 
\item Generative Adversarial Nets for inference of Individualized Treatment Effects (GANITE): We adapt the official package published by \citep{vanderschaarlab_ganite} to implement this method. The hyperparameters for the simulation implementation are set as default: the hidden dimensions $\textit{h\_dim}=100$, number of training iterations $\textit{num\_iterations} = 5000$, the batch size $256$, hyperparameters to adjust the loss importance $\alpha = 0.1$ and $\beta = 0$. 
\item Double Robust Forest Model (DR): We implement DR by the function $\textit{econml.dr.DRLearner}$ from the EconML package \citep{econml}. The propensity score is modeled by the function $\textit{sklearn.ensemble.RandomForestClassifier}$ in the package \citep{scikit-learn} with number of trees $\textit{n\_estimators} = 100$ and the maximum depth of the tree $\textit{max\_depth} = 2$. Both the outcome and the final model are implemented by the function $\textit{sklearn.ensemble.RandomForestRegressor}$ with both models have number of trees $\textit{n\_estimators} = 100$ and the maximum depth of the tree $\textit{max\_depth} = 2$ separately. We truncate the $\textit{min\_propensity} = 0.1$, which is the minimum propensity at which to clip propensity estimates to avoid dividing by zero.
\item Double Machine Learning Forest Model (DML): We implement DML by the function $\textit{econml.dml.CausalForestDML}$ from the EconML package \citep{econml}. The propensity score is modeled by the function $\textit{sklearn.ensemble.RandomForestClassifier}$ in the package \citep{scikit-learn} with number of trees $\textit{n\_estimators} = 100$ and the maximum depth of the tree $\textit{max\_depth} = 2$. The outcome is modeled by the function $\textit{sklearn.ensemble.RandomForestRegressor}$ with number of trees $\textit{n\_estimators} = 100$ and the maximum depth of the tree $\textit{max\_depth} = 2$ separately. 
\item Causal Forest (CF) on Covariates or Latent Factors: We implement CF by the function \textit{econml.grf.CausalForest} from the EconML package \citep{econml}. The parameters to implement the function are set as default: number of trees $\textit{n\_estimators} = 100$, the maximum depth of the tree $\textit{max\_depth} = 50$. 
\end{itemize}

\subsection{Additional tables and plots}

\begin{figure}[htbp]
	\centering
	\includegraphics[width=\textwidth]{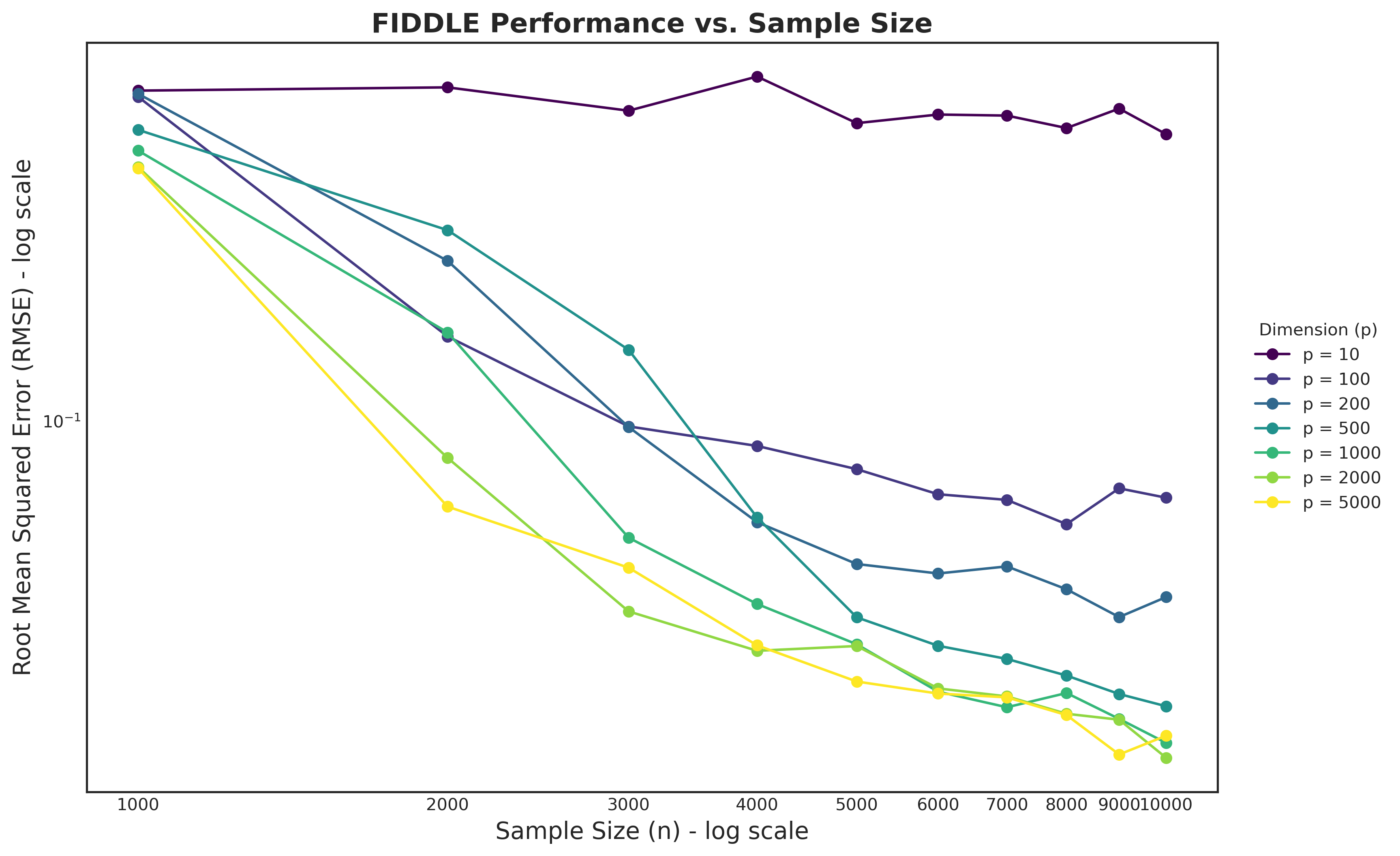}
	\caption{Plot of root mean squared error (RMSE) by FIDDLE performance across different sample sizes ($n$) and covariate dimensions ($p$) among $100$ replications (Both x and y axis plot on a log scale).} 
\label{fig:size_compare}
\end{figure}

\begin{table}[H]
	\centering
	\begin{tabular}{c|ccccccc}
		\hline
		& $p=10$ & $p=100$ & $p=200$ & $p=500$ & $p=1000$ & $p=2000$ & $p=5000$ \\
		\hline
		$n=1000$ & 0.5247 & 0.5087 & 0.5164 & 0.4310 & 0.3883 & 0.3577 & 0.3557 \\
		& (0.0342) & (0.0167) & (0.0166) & (0.0139) & (0.0130) & (0.0145) & (0.0134) \\
		\hline
		$n=2000$ & 0.5333 & 0.1529 & 0.2234 & 0.2606 & 0.1561 & 0.0832 & 0.0651 \\
		& (0.0299) & (0.0089) & (0.0079) & (0.0074) & (0.0086) & (0.0120) & (0.0150) \\
		\hline
		$n=3000$ & 0.4744 & 0.0974 & 0.0971 & 0.1430 & 0.0557 & 0.0385 & 0.0480 \\
		& (0.0251) & (0.0057) & (0.0040) & (0.0058) & (0.0037) & (0.0026) & (0.0106) \\
		\hline
		$n=4000$ & 0.5631 & 0.0883 & 0.0603 & 0.0617 & 0.0400 & 0.0316 & 0.0325 \\
		& (0.0282) & (0.0058) & (0.0031) & (0.0041) & (0.0025) & (0.0018) & (0.0025) \\
		\hline
		$n=5000$ & 0.4457 & 0.0786 & 0.0489 & 0.0374 & 0.0327 & 0.0324 & 0.0271 \\
		& (0.0239) & (0.0044) & (0.0030) & (0.0029) & (0.0022) & (0.0019) & (0.0018) \\
		\hline
		$n=6000$ & 0.4654 & 0.0693 & 0.0466 & 0.0324 & 0.0258 & 0.0262 & 0.0255 \\
		& (0.0231) & (0.0035) & (0.0026) & (0.0023) & (0.0018) & (0.0022) & (0.0018) \\
		\hline
		$n=7000$ & 0.4629 & 0.0674 & 0.0483 & 0.0304 & 0.0238 & 0.0252 & 0.0250 \\
		& (0.0240) & (0.0044) & (0.0037) & (0.0023) & (0.0017) & (0.0016) & (0.0018) \\
		\hline
		$n=8000$ & 0.4349 & 0.0597 & 0.0430 & 0.0279 & 0.0256 & 0.0231 & 0.0229 \\
		& (0.0213) & (0.0033) & (0.0026) & (0.0024) & (0.0019) & (0.0019) & (0.0017) \\
		\hline
		$n=9000$ & 0.4795 & 0.0714 & 0.0374 & 0.0255 & 0.0225 & 0.0224 & 0.0188 \\
		& (0.0242) & (0.0054) & (0.0026) & (0.0020) & (0.0014) & (0.0014) & (0.0014) \\
		\hline
		$n=10000$ & 0.4217 & 0.0681 & 0.0414 & 0.0239 & 0.0199 & 0.0185 & 0.0207 \\
		& (0.0200) & (0.0038) & (0.0021) & (0.0015) & (0.0013) & (0.0016) & (0.0013) \\
		\hline
	\end{tabular}
	\caption{Root mean squared error and standard error (in parentheses) of our proposed FIDDLE across different sample sizes ($n$) and covariate dimensions ($p$) over $100$ replications. For each $n$, the first row shows RMSE and the second row shows SE in parentheses.}
	\label{tab:size_compare}
\end{table}

{
	
	\bibliographystyle{apalike}
	\bibliography{reference}}

\end{document}